\documentclass[12pt]{article}%
\usepackage{amsmath}%
\usepackage{amsfonts}%
\usepackage{amssymb}%
\usepackage{float}
\usepackage{graphicx}
\providecommand{\U}[1]{\protect\rule{.1in}{.1in}}
\newtheorem{theorem}{Theorem}

\newtheorem{corollary}[theorem]{Corollary}

\newtheorem{lemma}[theorem]{Lemma}

\newenvironment{proof}[1][Proof]{\noindent\textbf{#1.} }{\ \rule{0.5em}{0.5em}}
\begin{document}

\title{The Faulty GPS Problem: Shortest Time Paths in Networks with Unreliable
Directions }
\author{Steve Alpern\\ORMS Group, WBS, University of Warwick}
\maketitle

\begin{abstract}
This paper optimizes motion planning when there is a known risk that the road
choice suggested by a Satnav (GPS) is not on a shortest path. At every branch
node of a network $Q$, a Satnav (GPS) points to the arc leading to the
destination, or home node, $H$ - but only with a high known probability $p.$
Always trusting the Satnav's suggestion may lead to an infinite cycle. If one
wishes to reach $H$ in least expected time, with what probability $q=q\left(
Q,p\right)  $ should one trust the pointer (if not, one chooses randomly among
the other arcs)? We call this the Faulty Satnav (GPS) Problem. We also
consider versions where the trust probability $q$ can depend on the degree of
the current node and a `treasure hunt' where two searchers try to reach $H$
first. The agent searching for $H$ need not be a car, that is just a familiar
example -- it could equally be a UAV receiving unreliable GPS information.

This problem has its origin not in driver frustration but in the work of Fonio
et al (2017) on ant navigation, where the pointers correspond to pheromone
markers pointing to the nest. Neither the driver or ant will know the exact
process by which a choice (arc) is suggested, which puts the problem into the
domain of how much to trust an option suggested by AI.

\end{abstract}

\section{Introduction}

A satellite navigation system (called Satnav, or GPS) suggests a road to take
at every intersection. More abstractly, it suggests an arc of the traffic
network $Q$ to take from any branch node. This arc is supposed to lie on the
shortest path to the destination, or Home node $H.$ Of course it is well known
that errors occur, so we model this by assuming that at every branch node
other than $H$ there is a pointer (to one of the incident arcs) which is
correct (goes along a shortest path) with a known probability $p$ called the
\textit{reliability }- otherwise it points to a random incorrect arc. The set
of pointers are fixed throughout the journey, so if a node is encountered
several times, the pointer will always suggest the same arc. If one always
follows the pointer, one may cycle infinitely and never reach the destination.
More generally, always following the pointer may not minimize the total travel
time to $H.$ There are many ways that a real life driver deals with this
problem. She might remember that taking a particular arc at an earlier
occasion at the current node led back to it, so she might try something
different the second time. We will not consider this or other advanced
techniques that an alert driver with a good memory might use. Rather, we adopt
a simple model of the driver (or autonomous vehicle navigation program). We
assume a simple \textit{trust probability} $q$ with which to follow the
pointer. The question we consider is to how to optimize the trust $q$ to
minimize time to $H,$ given the initial and final nodes $I$ and $H,$ the
network $Q,$ and the reliability $p.$ With a probability equal to the trust
value $q,$ the arc indicated by the pointer is taken; otherwise one of the
other arcs is chosen randomly. We call the question of optimizing trust The
\textit{Faulty GPS Problem}. After formalizing the problem in Section 2, we
present in Section 3 a \textit{slow method} of solving three particular
networks: a triangle, a circle-with-spike and a simple tree. We then develop a
general theory for (i) stars (Section 4), (ii) networks with bridges (Section
5) and (iii) trees (Section 6). In Section 7 we determine how long it takes to
cross a line graph with varying length arcs. In Section 8 we consider briefly
small cycle graphs of odd and even lengths. In Section 9 we consider a game
theoretic \textit{treasure hunt} version of the problem, in which two drivers
with the same GPS system try to be the first one to reach the destination $H.$
We solve this game on a very simple line graph, in the cases where the drivers
start at the same or different node. Section 10 concludes.

It is worth mentioning that our Satnav metaphor is only that, a metaphor to
simply describe the problem of shortest paths with unreliable directions. Real
Satnav errors are not likely to be random as assumed here, but rather only
suboptimal and generally still pointing in a good direction. We are not
recommending our strategies to drivers! In fact the real life problem that
motivated this paper comes from biology, as described in the next paragraph.

This problem has its origin not in a driver GPS setting, but in a study of how
a species of ants navigates back to their nest, carried out by Fonio et al
(2016). They settled a long standing question by chemical analysis of
pheromones laid by individual ants, showing that these deposits formed a
decentralized system of pointers towards the nest. There and in Boczkowski,
Korman, A and Rodeh, Y. (2018) a deep computer science analysis of query
complexity and move complexity is carried out on unit tree networks (all arcs
have unit length). It should be observed that for deterministic shortest path
problems, a solution for unit networks could be easily applied to general ones
by the insertion of additional degree two nodes at regular intervals. However
in the Faulty GPS Problem such additional nodes greatly increase travel times,
since we have not precluded backtracking. So considering networks of general
arc lengths (as well as cycles) is required.

In our problem, the driver (searcher) does not see the whole network, only the
node he currently occupies and its incident arcs. In this respect the problem
is similar to the maze problem of Gal and Anderson (1990). There also, the
searcher adopts a randomized strategy for leaving the current node. However
instead of a pointer, the searcher has available markings he is allowed to
make on earlier visits to that node. Allowing such marking in our problem is
an interesting variation for future work, as it corresponds to driver memory
alluded to above. More generally, the problem of finding the destination node
$H$ could be seen as a network search problem. If Nature is viewed as
antagonistic, the game models from Gal (1979) up to the discrete arc-choice
model of Alpern (2017) could be seen as related. If the problem of inaccurate
directions at a node can be thought of as a sort of search cost at the node,
then the model of Baston and Kikuta (2013) is related. The game theoretic
analysis of Section 9 follows the first-to-find paradigm of Nakai (1986) and
Hohzaki (2013) and Duvocelle et al (2017) and is similar to the
winner-take-all game of Alpern and Howard (2017).

The problem on the line graph treated in Section 8 has similarities with what
is known as dichotomous, or high-low, search on the line. After each move of
arbitrary size along the line, the searcher is told the direction but not the
distance to the target location $H,$ which the searcher wishes to find in the
least number of moves or some related efficiency measure. Unlike the current
version, going past the target does not solve the problem. In some
applications, the target is the demand for a product, as in the newsboy
problem. Some of the original papers in this area are Baston and Bostock
(1985), Alpern (1985), Alpern and Snower (1988) and Reyniers (1990), as
surveyed in Hassin and Sarid (2018). Computer scientists have worked on
related problems from a different point of view, for example Miller and Pelc (2015).

From a more abstract AI perspective, the problem addressed in specific form
here is how much to trust a course of action suggested by a process such as
GPS planning, when the exact algorithm underlying the process is not known.

\section{The Satnav (GPS) Problem}

This section formalizes the Satnav Problem on a network $Q.$ The network $Q$
has a node set $\mathcal{N}$ and a distinguished home node $H\in\mathcal{N}$
which represents the nest (for the ant probem) or the destination (in the
satnav interpretation). The arcs $e$ of $Q$ have given lengths $\lambda\left(
e\right)  $. The branch nodes $\mathcal{N}_{\mathcal{B}}$ of $\mathcal{N}$ are
the nodes of degree at least $2$ other than the destination $H$ itself. At
these nodes the agent who wants to get home must make a decision as to which
arc to take next. A \textit{direction vector} $d:\mathcal{N}_{\mathcal{B}%
}\rightarrow\mathcal{N}$ tells the agent which arc to take. So for each branch
node $i\in\mathcal{N}_{\mathcal{B}},$ the arc $d\left(  i\right)  $ is is
specified by giving a node $j=d\left(  i\right)  $ which is adjacent to $i.$
Alternatively we can specify the arc incident to node $i$. (For example if two
arcs lead to the same node we must specify the arc, but this is unusual, and
we can exclude multiple arcs if we wish.) So we think of $d\left(  i\right)  $
as an arc incident to node $i,$ as the agent doesn't know which node it leads
to. The set of direction vectors is denoted by $\mathcal{D},$ and a measure
$\mu$ on $\mathcal{D}$ is defined as follows ($\mu\left(  d\right)  $ is the
probability that a Satnav with reliability $p$ chooses the direction vector
$d$): Each arc $d\left(  i\right)  $ is chosen independently: with a given
probability $p$ (called the \textit{reliability}) an arc on a shortest path to
$H$ is chosen randomly (generically such an arc is unique); with complementary
probability $1-p$ one of the other arcs is randomly chosen. A simple strategy
for the Searcher is always to choose the arc $d\left(  i\right)  $ with a
fixed probability $q,\mathit{\ }$called the \textit{trust} (or trust
probability). Following such a strategy, the expected time to reach the home
node $H$ from the initial node $I$ is denoted by $T.$ To indicate that $T$ is
the travel time from one node to another, we also may write $T=T\left(
A,B\right)  $ for the time from $A$ to $B.$ When the home node $H$ is fixed,
we can write $T_{A}=T\left(  A,H\right)  $ for simplicity of notation. The
Satnav Problem is to minimize $T$ by choosing the optimal \textit{trust
probability} (or just \textit{trust}) $\hat{q}=\hat{q}\left(  p\right)  .$ For
a given direction vector $d,$ the trust $q$ determines a Markov chain on the
nodes $\mathcal{N}$ of $Q,$ with absorbing node $H,$ and has an expected
hitting time $T^{d}\left(  A\right)  $ from every possible starting node $A.$
The time $T$ is an average time over all direction vectors%
\begin{equation}
T\left(  A,H\right)  =T_{A}=\sum\limits_{d\in\mathcal{D}}\mu\left(  d\right)
~T^{d}\left(  A\right)  .\label{T sum of Td}%
\end{equation}
Note that $T\left(  A,H\right)  $ is a function of $p$ and $q,$ as $\mu\left(
d\right)  $ is a function of $p$ and $T^{d}\left(  A,H\right)  $ is a function
of $q.$

We will make the general assumption that $Q$ has no loops or multiple arcs and
shortest paths are unique. Actually we can deal with the last two in some
cases. We can also assume that $H$ is not a cut node.

We also consider a \textit{counting agent} variation. This assumes that the
searching agent, on reaching a node, can count how many arcs are incident at
the node (he knows the degree of the node), and can choose to follow the
direction at a node with a probability $q_{k}$ where $k$ is the degree of the
node. In this variation the choice variable is the vector $\left(
q_{k}\right)  $ where $k$ varies over the degrees of the branch nodes of the
network. Sometimes we consider the optimization problem for the trust at a
single node, when trusts at all other nodes are fixed.

Our model is illustrated in Figure 1, where we show a network with a
destination node $H$ and a direction vector (solid arrows at branch points).
We add a dashed arrow at the upper left leaf node to indicate that one always
reflects from that node. The correct pointers (leading to shortest paths to
$H$) are in green, the incorrect ones in red. (These colors are for the
reader, not for the searcher.) Note that if $q=1$ and one gets to the (top
left) leaf arc, then one never leaves it. Similarly, if one never follows the
arrows, $q=0,$ then the leaf arc at $H$ is never taken (this is a simple case
of the argument in Lemma 4).%
    \begin{figure}[H]
     \centering
     \includegraphics[width=0.85\textwidth]{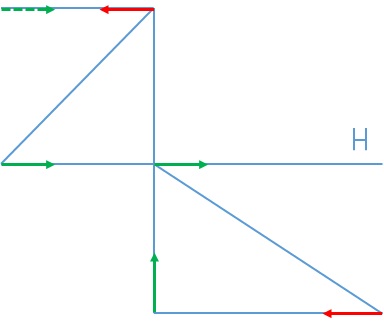}
     \caption{A direction 6-vector, possibly for $p=2/3.$}
    \end{figure}
\section{Examples}

Before developing any theory, we first introduce three simple examples which
show how the GPS problem can be solved by a `slow method'. In some cases we
will see later how the analysis can be simplified by more general theory
developed later. Our main interest is how, in each example, the optimal trust
probability $\hat{q}$ depends on the reliability $p.$ In some cases we also
consider the counting searcher problem.

\subsection{A triangle network}

Consider the network $Q_{1}$ pictured below in Figure 2. We will generally
take the length $x$ of the third side to be $3$ to give exact values, but
giving an arbitrary length shows the effect of arc length on the solution to
this problem.%
    \begin{figure}[H]
     \centering
     \includegraphics[width=0.65\textwidth]{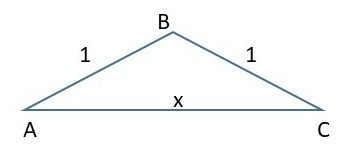}
     \caption{Triangle network}
    \end{figure}
We consider the above triangle with $I=A$ as the starting node and $H=C$ the
home, or destination, node. Let $+$ and $-$ denote clockwise and
anti-clockwise directions for pointers. Take $x>2$ so that the correct
directions (shortest paths to $C)$ are $+$ at $A$ and $+$ at $B.$ There are
four possible direction vectors, which we label as $d^{1}=\left(  +,+\right)
$ (both correct), $d^{2}=\left(  -,-\right)  $ (both wrong), $d^{3}=\left(
+,-\right)  $ (correct at $A,$ wrong at $B)$ and $d^{4}=\left(  -,+\right)  $
(wrong at $A,$ correct at $B$). Their respective probabilities are given by
$p^{2},\left(  1-p\right)  ^{2},$ $p~\left(  1-p\right)  ,$ and $\left(
1-p\right)  ~p.$ Let $a_{i}=T^{d^{i}}\left(  A,C\right)  $ and $b_{i}%
=T^{d^{i}}\left(  B,C\right)  $ denote the expected times to $C$ from $A$ and
from $B$ for trust probability $q$ and direction vector $d^{i}.$ Note that $p$
does not yet come into these probabilities.

When $d=d^{1}$ we have the equations
\begin{align*}
a_{1}  & =q\left(  1+b_{1}\right)  +\left(  1-q\right)  \left(  x\right)
,~b_{1}=q\left(  1\right)  +\left(  1-q\right)  \left(  1+a_{1}\right)
,\text{ so}\\
a_{1}  & =\allowbreak\frac{2q+x-qx}{q^{2}-q+1}\text{ and }b_{1}=\frac
{q-\left(  q-1\right)  \left(  q-x\left(  q-1\right)  +1\right)  }{q^{2}-q+1}.
\end{align*}
Similarly we have the formulae%
\begin{align*}
a_{2}  & =\frac{-2q+qx+2}{-q+q^{2}+1},~b_{2}=\frac{q+q^{2}x-q^{2}+1}%
{-q+q^{2}+1}~\\
a_{3}  & =\frac{2q+x-qx}{(1+q)(1-q)},b_{3}=\frac{-q^{2}x+qx+q^{2}%
+1}{(1+q)(1-q)},\text{and}\\
a_{4}  & =\frac{-2q+qx+2}{q\left(  2-q\right)  },~b_{4}=\frac{-2q-q^{2}%
x+qx+q^{2}+2}{q\left(  2-q\right)  }.
\end{align*}

Thus the expected time from $A$ to $C=H$ is given by $T_{A}=T\left(
A,C\right)  .$
\[
T\left(  A,C\right)  =p^{2}~a_{1}\left(  q\right)  +\left(  1-p\right)
^{2}~a_{2}\left(  q\right)  +p\left(  1-p\right)  ~a_{3}\left(  q\right)
+p\left(  1-p\right)  ~a_{4}\left(  q\right)  ,
\]
with a similar formula for $T\left(  B,C\right)  .$ We can see in Figures 3
and 4 how $T_{A}$ and $T_{B}$ vary with $q$ when $x=3$ and $p=3/4=.75$ and
$p=0.96.$ For $p=.75,$ starting from $A$ the optimal $q$ is about 0.68 and
from $B$ it is about 0.72, while for $p=.96$ the optimal $q$ is about $.885$
starting at $A$ and about $.879$ starting at $B.$ Thus the order has reversed.
We numerically calculate a value of $\tilde{p}\simeq0.925$ when there is a
uniformly optimal trust value $\hat{q}\simeq0.84$ (optimal for any start). For
$p<$ $\tilde{p},$ we have $\hat{q}\left(  A\right)  $ (starting at $A$%
)$<\hat{q}\left(  B\right)  ,$ while for $p>\tilde{p}$ we have $\hat{q}\left(
A\right)  $ $>\hat{q}\left(  B\right)  .$\\ \\
\begin{minipage}[c]{0.5\linewidth}
    \begin{figure}[H]
     \centering
     \includegraphics[width=0.95\textwidth]{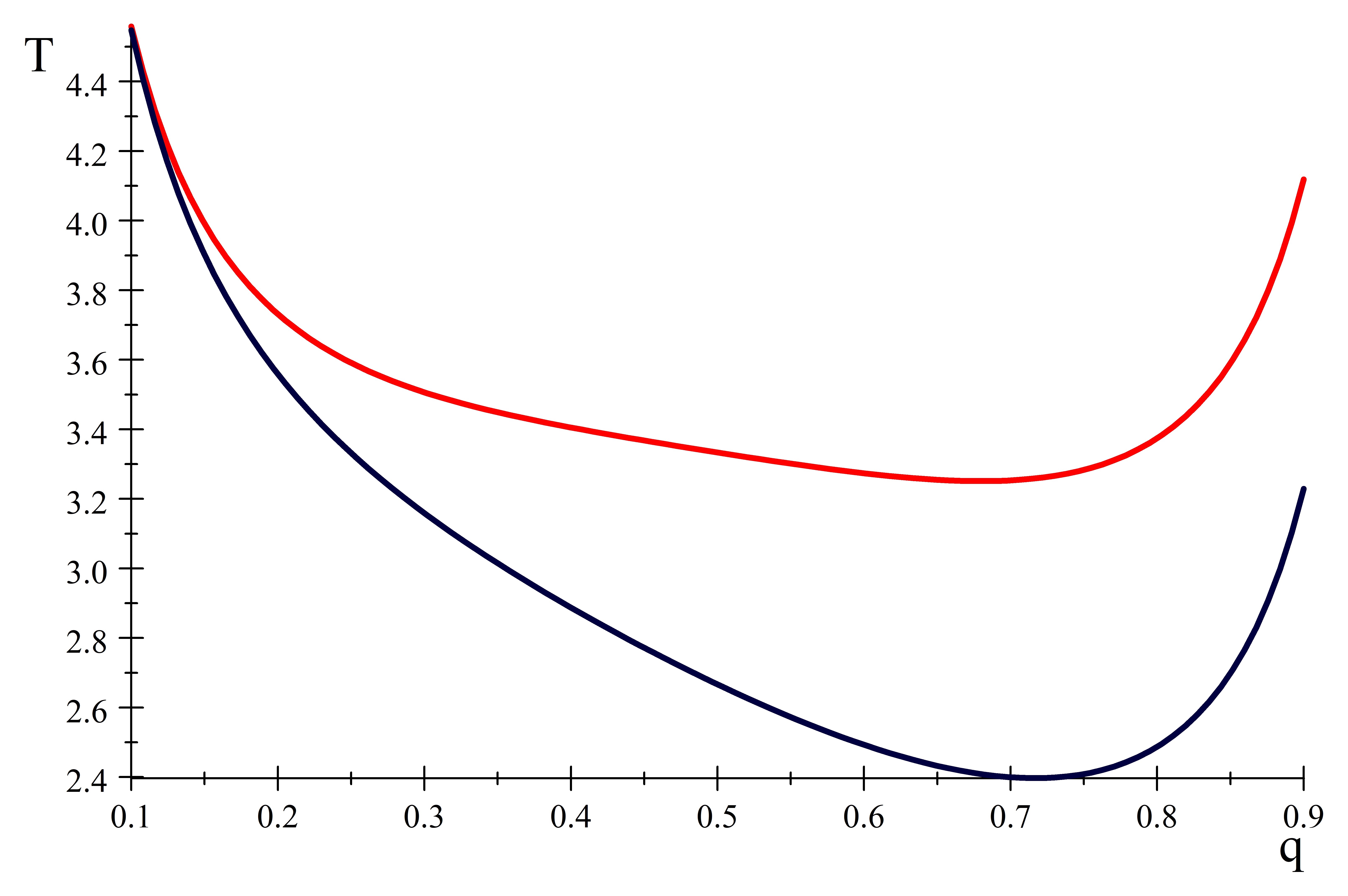}
     \caption{$T_A$ (top), $T_B,$ $p=3/4.$}
    \end{figure}
\end{minipage}
\begin{minipage}[c]{0.5\linewidth}
    \begin{figure}[H]
     \centering
     \includegraphics[width=0.95\textwidth]{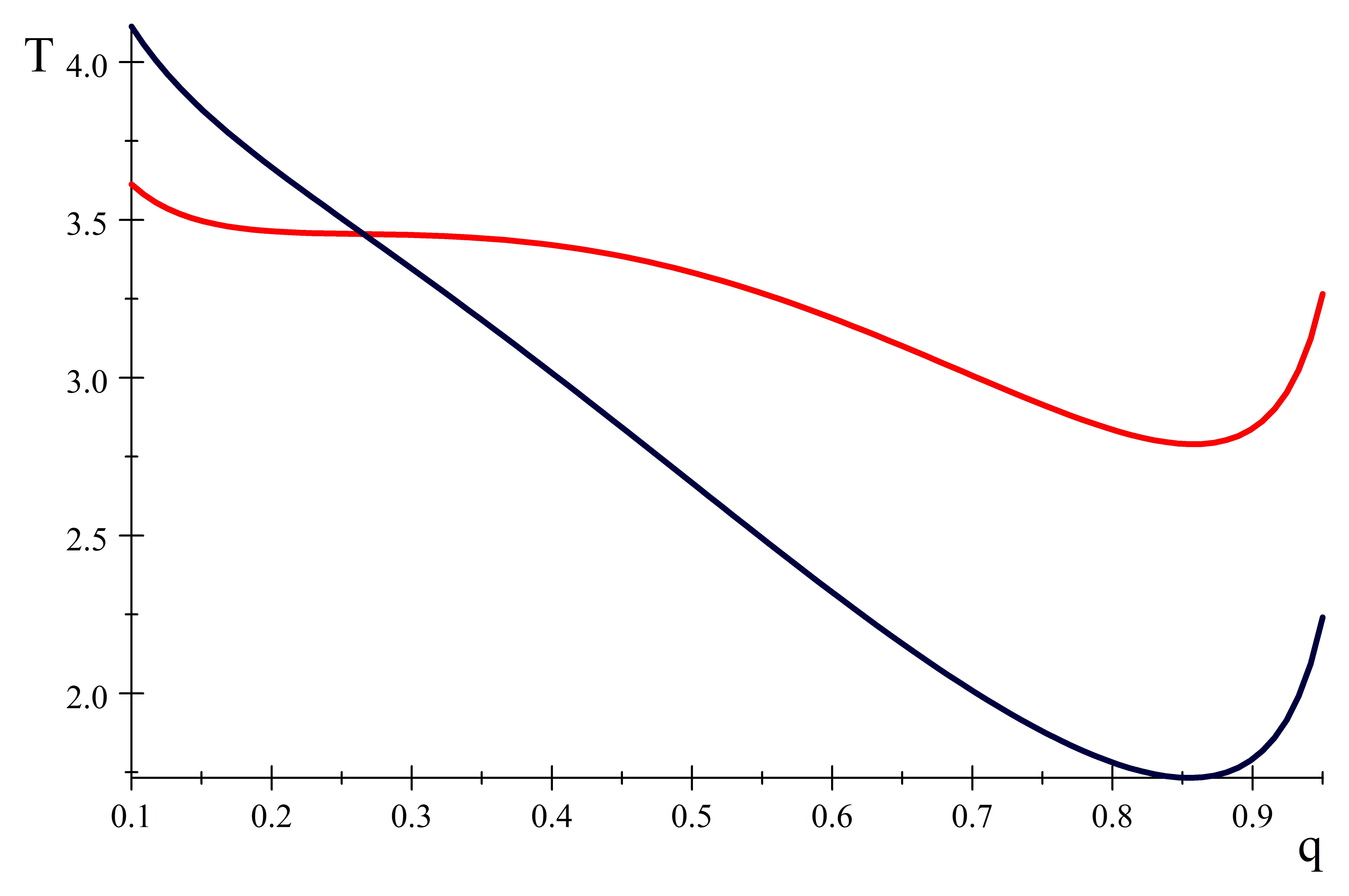}
     \caption{$T_A$ (top right), $T_B,$ $p=0.96.$}
    \end{figure}
\end{minipage}\\ \\ \\
Note that $a_{3}$ has the factor $\left(  1-q\right)  $ in its denominator and
hence goes to infinity when $q$ goes to $1;$ similarly $a_{4}$ has the factor
$q$ in its denominator and hence goes to infinity when $q$ goes to $0. $ This
observation can also be based on the cycle $ABAB...$ which will go on for a
long time in these cases. A more generalizable argument is based on the
observation that $C$ can be reached only by traversing one of the arcs $AC$ or
$BC.$ If both of these are directed (by pointers) at $A$ and $B$ away from $C$
(the pointer vector we called $d^{3})$ then if $q=1$ the Home node $H=C$
cannot be reached. A similar argument works for $d^{4}$ with $q=0.$ So this is
a good place to state the following easy generalization.

\begin{theorem}
Fix $Q,I,H,p\in\left(  0,1\right)  .$ and let $T\left(  q\right)  $ denote the
expected time to reach $H$ from $I$ with trust $q.$ Then we have $T\left(
q\right)  \rightarrow\infty$ as $q\rightarrow0$ or $1.$ Hence $T\left(
q\right)  $ has an interior minimum $\hat{q}\in\left(  0,1\right)  .$
\end{theorem}

\begin{proof}
Suppose the direction vector $d^{\ast}$ is such that at every vertex adjacent
to $H$, it points to $H.$ When reaching such a vertex (or if starting there),
one has to follow the pointer eventually to reach $H$. The expected number of
times this takes is $1/q,$ so the expected time $T$ is at least $L/q$, where
$L$ is the smallest edge length. So $T\geq\mu\left(  d^{\ast}\right)  ~L/q,$
which goes to $\infty$ as $q\rightarrow0.$ For $q\rightarrow1$ as similar
result holds for any direction vector $d^{\ast\ast}$ in which at every vertex
adjacent to $H,$ it doesn't point to $H.$ Note that for any fixed $p\in\left(
0,1\right)  $ and direction vector $d$ , $T^{d}\left(  q\right)  $ is a family
of Markov chains with the same absorbing state $H$ and hence the hitting time
of $H$ is continuous in $q.$ It follows that $T$ has an interior minimum
$\hat{q}\in\left(  0,1\right)  .$
\end{proof}

\subsection{The circle with spike network}

We now apply the slow method to the circle-with-spike network shown in Figure
5. This graph has multiple edges but that does not give us any problems. There
is six direction vectors (two choices at $A,$ three at $X$). We apply the same
`slow method' as for the previous example, leaving out the details.
    \begin{figure}[H]
     \centering
     \includegraphics[width=0.85\textwidth]{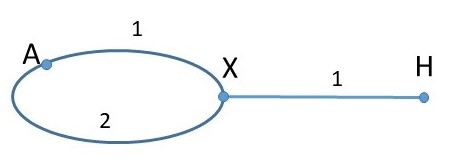}
     \caption{Circle-with-spike networks.}
    \end{figure}
We note that the shortest path from $A$ to $H$ goes along the arc of length
$1$ and from $X$ it goes along the arc to $H$. First consider the `counting
searcher' version mentioned in the Introduction, where the trust probability
$q_{j}$ is allowed to depend on the degree $j$ of the current node. Let
$r=q_{2}$ denote the trust probability at $A$ and $s=q_{3}$ denote the trust
probability at $X.$ Using the same simultaneous equation method as in the last
subsection, and averaging over the six direction vectors, we find
\begin{gather}
T\left(  A,X\right)  =1+p+\left(  1-2p\right)  r~\text{and also }T\left(
X,H\right)  \text{ is given by}\label{a}\\
\frac{2p^{2}(-1+2r)(-1+3s)+s(7+3s+2r(1+s))-p(-5+13s-2s^{2}+2r(-1+5s+2s^{2}%
)))}{2(1-s)s}\nonumber
\end{gather}
From $A$, since both arcs lead to $X$, one should take the one most likely to
be the short arc. So the optimum $r=q_{2}$ is 1 when $p>1/2$ and 0 when
$p<1/2,$ which can also be seen from (\ref{a}). It is easily calculated that
for $p=3/4$ the counting agent problem for $T\left(  X,H\right)  $ is
minimized at about $5.056$ with $r=q_{2}=1$ (this is true more generally for
$p>1/2$) and $s=q_{3}\simeq0.55051,$ starting at either node. Later we will
show how the counting problem can be solved more easily by considering an
associated star network and applying the theory for stars developed in the
next section. For the original (non-counting searcher) the solution starting
from $X$ has $r=s=q$ minimized at $5.38$, with $q\simeq0.56.$ Starting from
$A,$ the time $T\left(  A,H\right)  $ to $H$ is minimized at $6.85$, with
$q\simeq$ $0.57108.$

\subsection{A tree with two branch nodes}

Consider the tree network drawn in Figure 6, which has two leaf nodes $1$ and
$2,$ two branch nodes $A$ and $B,$ and a destination node $H.$ All arcs have
unit length. The pointing directions are indicated for clarity (for example to
A,1 or H at node B). For the basic problem (single trust probability for all
nodes), we let $q_{A}$ denote the trust everywhere when starting at $I=A,$ so
$\hat{q}_{A}$ minimizes $T_{A}=T\left(  A,H\right)  $ and similarly for
$q_{B}.$ For the counting searcher problem, let $q_{2}$ and $q_{3}$ denote the
trust at $A$ (degree 2) and at $B$ (degree 3) when these are allowed to be
different.
    \begin{figure}[H]
     \centering
     \includegraphics[width=0.85\textwidth]{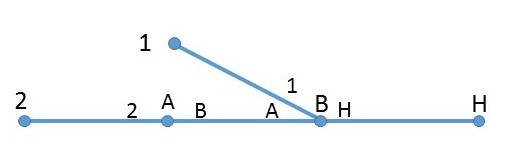}
     \caption{A simple tree with two branch nodes.}
    \end{figure}
We begin by considering the problem with a `counting' agent' taking $r=q_{2}$
(trust at $A)$ and $s=q_{3}$ (trust at $B).$ For the `slow method' of Section
3.1 we have to first consider the six direction vectors in $\left\{
2,B\right\}  \times\left\{  A,1,H\right\}  ,$ the two directions at $A$ and
the three at $B.$ For the direction vector $d^{\ast}=\left(  B,1\right)  ,$
the travel times $T_{A}^{d^{\ast}}=T^{d^{\ast}}\left(  A,H\right)  $ and
$T_{B}^{d^{\ast}}=T^{d^{\ast}}\left(  B,H\right)  $ satisfy
\begin{align*}
T_{A}^{d^{\ast}}  & =r\left(  1+T_{B}^{d^{\ast}}\right)  +\left(  1-r\right)
\left(  2+T_{A}^{d^{\ast}}\right)  ,\text{ }T_{B}^{d^{\ast}}=\frac{1-s}%
{2}\left(  1+\left(  1+T_{A}^{d^{\ast}}\right)  \right)  +s\left(
2+T_{B}^{d^{\ast}}\right)  ,~\text{so}\\
T_{A}^{d^{\ast}}  & =\frac{-4s+4rs+4}{r-rs}\,\text{and }T_{B}^{d^{\ast}}%
=\frac{r-2s+3rs+2}{r-rs}.
\end{align*}

Using analogous methods for the five other direction vectors $d,$ and then
averaging them with weights $\mu\left(  d\right)  ,$ we obtain the formulae
for $T\left(  A,H\right)  $ (top line) and $T\left(  B,H\right)  $ (bottom
line) as follows%

\begin{align*}
& \frac{\left(  2r+3s-6rs-1\right)  p^{2}+\left(  r^{2}-3r^{2}s-2rs^{2}%
+12rs-2r+s^{2}-3s\right)  \allowbreak p+\left(  r^{2}s^{2}+r^{2}s-4rs\right)
}{rs\left(  1-r\right)  \left(  s-1\right)  }\allowbreak\\
& \frac{\left(  6rs-3s-2r+1\right)  p^{2}+\left(  3r^{2}s-r^{2}-2rs^{2}%
-8rs+2r+s^{2}+s\right)  \allowbreak p+\left(  rs^{2}-2r^{2}s+3rs\right)
}{rs\left(  r-1\right)  \left(  s-1\right)  }%
\end{align*}

For the original non-counting problem we set $q=r=s$ to minimize
$T_{A}=T\left(  A,H\right)  $ and then $T_{B}=T\left(  B,H\right)  $ over $q.$
In Figure 7 we fix reliability at $p=3/4$ and plot the expected times to reach
$H$ from $A$ (lowered by $2.78$ to fit in picture) and from $B.$ It can be
seen that the optimal trust (at all nodes) when starting at $A$ is
approximately .59, which is higher than the optimal trust of about .57 when
starting at $B.$ We mention this, because later we shall show that for the
line graph the optimal trust probability does not depend on the starting node,
there is a uniformly optimal trust. At these respective optimal trusts, we
have $\hat{T}\left(  A,H\right)  \simeq\allowbreak8.\,\allowbreak05$ and
$\hat{T}\left(  B,H\right)  \simeq\allowbreak5.28.\,\allowbreak$ \\
    \begin{figure}[H]
     \centering
     \includegraphics[width=1\textwidth]{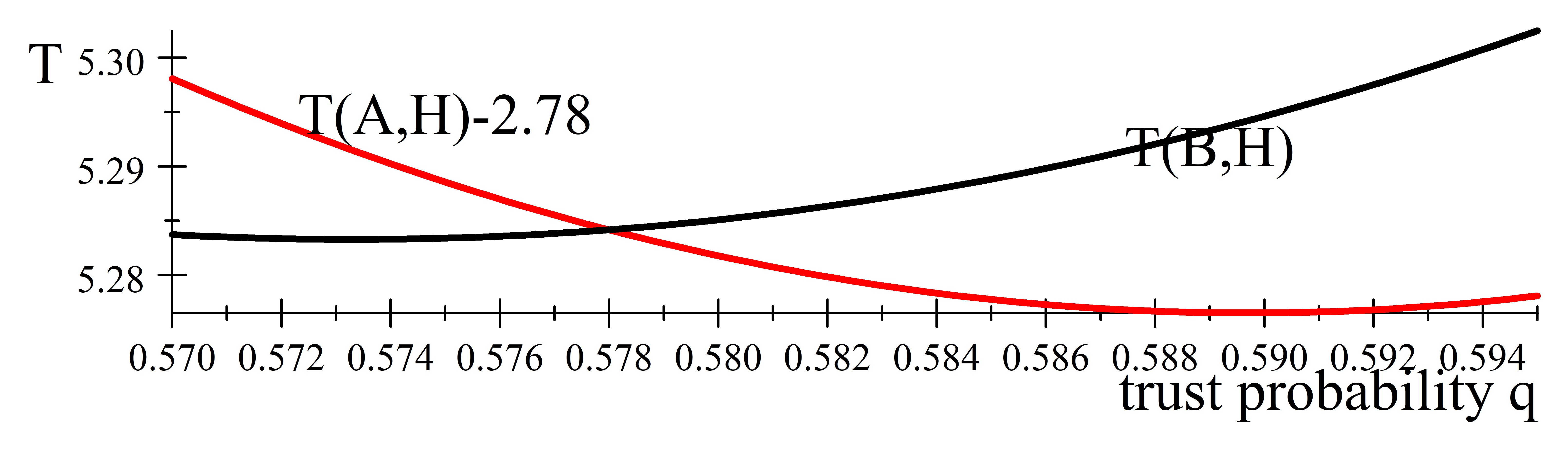}
     \caption{Plots of time to $H$ from A,B for $p=3/4.$}
    \end{figure}
    \newpage
More generally, we set $r=s=q$ and calculate%

\begin{align*}
\frac{\partial T_{A}}{\partial q}  & =\frac{\allowbreak\left(  3-5p\right)
q^{4}+\left(  23p-12p^{2}-7\right)  q^{3}+\left(  15p^{2}-15p\right)
\allowbreak q^{2}+\left(  5p-9p^{2}\right)  q+2p^{2}}{(-1+q)^{3}~q^{3}},\\
\frac{\partial T_{B}}{\partial q}  & =\frac{\left(  1-p\right)  q^{4}+\left(
15p-12p^{2}-5\right)  q^{3}+\left(  15p^{2}-9p\right)  \allowbreak
q^{2}+\left(  3p-9p^{2}\right)  q+2p^{2}}{(-1+q)^{3}q^{3}}.
\end{align*}
Setting fourth degree polynomials in the numerators to zero, we obtain
implicit functions for $\hat{q}_{A},\hat{q}_{B}$ as functions of $p,$ which we
plot as the two middle curves in Figure 8. We see that $\hat{q}_{A}>\hat
{q}_{B}$ for all $p,$ $0<p<1.$

Finally, we consider the counting agent problem, where we can jointly minimize
$T_{A}$ and $T_{B}$ with $r=q_{2}$ and $s=q_{3}$ ($A$ has degree 2, $B$ has
degree 3). For our comparison base $p=3/4,$ the optimal values of trust are as
follows: when the search agent can count the degree of a node, he can reach
$H$ in expected time about $7.96$ from $A$ and $5.23$ from $B;$ in both cases
adopting trust $\hat{q}_{2}=\frac{3}{2}-\frac{1}{2}\sqrt{3}\allowbreak
\simeq\allowbreak0.634\,$when at $A$ and $\hat{q}_{3}=3-\sqrt{6}%
\allowbreak\simeq0.551$ when at $B.$ In the main case, where he cannot count
and must trust equally at all branch nodes, he trusts with probability
$q_{A}\simeq\allowbreak0.590$ at both nodes when starting at $A,$ reaching $H$
in expected time $8.057$. When starting from $B,$ he reaches $H$ in expected
time 5.283, trusting with probability $\bar{q}_{B}\simeq$ 0.573 at both nodes.
The four trust probabilities for this network are shown below in Figure 8 The
important thing to note, probably with general applicability, is that that
when counting degree the searcher can use more extreme trust values, but when
trust has to be the same at all nodes, less extreme trust values must be
adopted.
    \begin{figure}[H]
     \centering
     \includegraphics[width=0.95\textwidth]{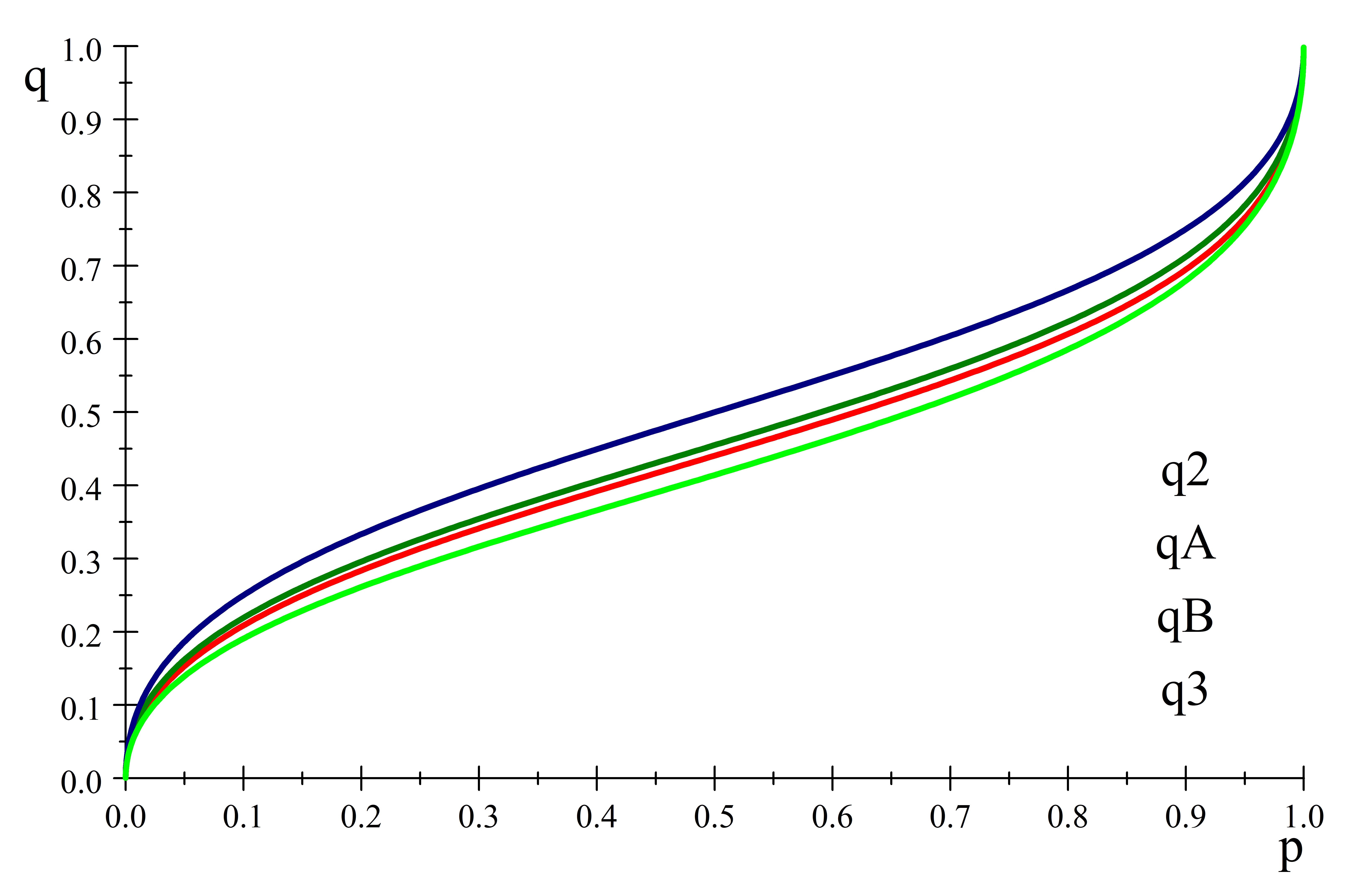}
     \caption{Plots of optimal trusts $q_2,q_A,q_B,q_3.$}
    \end{figure}
\section{ Star Networks}

We now consider a star network $Q^{n}$ where one of the $n$ rays (leaf arcs)
leads to the home node $H$ and the start node $I$ is the central node. It
turns out that the optimal trust probability $\hat{q}=\hat{q}_{n}$ (for the
single branch node $I$) depends only on $p$ and the degree $n$ of the central
node. The lengths of the rays do not matter, though of course they affect the
optimal travel time. Our analysis of the star will have implications for other
networks, because locally every node is a star.

\begin{theorem}
\label{star}Let $Q$ be a star network with a single branch node $I$ (the
center node) of degree $n.$ Assume that the home node $H$ is one of the leaf
nodes, with a leaf arc of length $c.$ The other $n-1$ rays (arcs)
$i=1,\dots,n-1,$ have lengths denoted by $\alpha_{i},$ whose sum is denoted
$\alpha=\sum\limits_{i=1}^{n-1}\alpha_{i}.$ The expected time $T$ to get to
$H$ from $I$ is given by%
\begin{equation}
T=c+\frac{\allowbreak\left(  2p-4q+2q^{2}+2nq-2npq\right)  }{q\left(
1-q\right)  \left(  n-1\right)  }\alpha,\label{Tstar}%
\end{equation}
which is minimized by taking $q=\hat{q}_{n}$ to be%
\begin{align}
\hat{q}_{n}  & =\bar{q}_{n}\left(  p\right)  \equiv\frac{p-\sqrt{n-1}%
\sqrt{p\left(  1-p\right)  }}{1-n\left(  1-p\right)  },\text{ for }p\neq
\frac{n-1}{n}\text{ and}\label{qn bar star}\\
\hat{q}_{n}  & =\bar{q}_{n}\left(  1/2\right)  \equiv\frac{1}{2},\text{ for
}p=\frac{n-1}{n},
\end{align}
independent of the lengths of the rays.
\end{theorem}

\begin{proof}
\bigskip Since there is a single branch node $I,$ the direction vector has a
single element which we call just $d.$ If $d=h$ (points to $H),$ we calculate
the time $T^{h}$ to reach $H,$ using
\begin{align}
T^{h}  & =q\left(  c\right)  +\frac{1-q}{n-1}\sum\limits_{i=1}^{n-1}\left(
2\alpha_{i}+T^{h}\right) \label{TH}\\
& =qc+\left(  1-q\right)  \left(  T^{h}+2\alpha/\left(  n-1\right)  \right)
,\text{ or}\nonumber\\
T^{h}  & =\frac{1}{q}\left(  cq-2\frac{\alpha}{n-1}\left(  q-1\right)
\right)  .\nonumber
\end{align}
\bigskip If $d=i$ points along one of the other rays $i=1,\dots,n-1\,,$ then
the time $T^{d}=T^{i}$ to reach $H$ satisfies the equation%
\begin{align}
T^{i}  & =q\left(  2\alpha_{i}+T_{i}\right)  +\left(  1-q\right)  \left(
1/\left(  n-1\right)  \right)  \left(  c+\sum\limits_{j\neq i}\left(
T^{i}+2\alpha_{j}\right)  \right)  ,\text{ }\label{T_i}\\
& =q\left(  2\alpha_{i}+T^{i}\right)  +\left(  \left(  1-q\right)  /\left(
n-1\right)  \right)  \left(  c+\left(  n-2\right)  T^{i}+2\left(
\alpha-\alpha_{i}\right)  \right)  ,~\text{or}\nonumber\\
T^{i}  & =\frac{2\alpha-2\alpha_{i}+c-2\alpha q-cq+2\alpha_{i}nq}%
{1-q}.\nonumber
\end{align}
So the overall time to reach $H$ is given by $T=pT^{h}\sum\limits_{i=1}%
^{n-1}\frac{\left(  1-p\right)  }{n-1}T^{i},$ $\allowbreak$which simplifies to
(\ref{Tstar}), as claimed. To find the optimal trust, it is enough to solve
the first order condition%
\[
2p-4q+2q^{2}+2nq-2npq=0\text{ }%
\]
which gives the optimal trust $\hat{q}=\bar{q}_{n}$ of (\ref{qn bar star}).In
the case where the denominator of (\ref{qn bar star}) is zero, the limiting
value of $1/2$ is obtained by L'Hospital's rule.
\end{proof}

Some values of $\hat{q}\left(  n,p\right)  $ for our standard reliability of
$p=3/4$ are given in Table 1. Note that the value for $n=3$ (given to more
places) is the same as we found using the slow method for the degree three
node $X$ in the circle-with-spike graph of Figure 5, which is not a star. An
exact analysis of the circle-with-spike network will be given in the next section.%

\begin{gather*}%
\begin{tabular}
[c]{|l|l|l|l|l|l|l|}\hline
$n$ & 2 & 3 & 4 & 5 & 6 & 7\\\hline
$\bar{q}\left(  n,p\right)  $ & 0.634 & 0.55051 & 0.500 & 0.464 & 0.436 &
0.414\\\hline
\end{tabular}
\\
\text{Table 1. Trust values }\hat{q}_{n}=\text{ }\bar{q}\left(  n,p\right)
,\text{ }p=3/4,n=2,\dots,7.
\end{gather*}

\section{Graphs With Bridges}

In the Star Theorem (Theorem \ref{star}), a searcher leaving the central node
$I$ via an arc not leading to $H$ will come back immediately from the
corresponding leaf node to $I.$ It turns out that the same method of analysis
works as long as when leaving $I$ by such an arc, the searcher returns to $I$
\textit{before reaching} $H.$ This property can be ensured by specifying that
$IH$ is a bridge arc. Note that the circle-with-spike network of Figure 5 has
this property for $I=X.$ A more general version of the Star Theorem, which for
example applies immediately to that network, can then be stated as follows.

\begin{theorem}
[Bridge Theorem]\label{localstar}Let $XH$ be a bridge (disconnecting) arc of
length $c$ of a network $Q,$ with the degree of $X$ equal to $n.$ Assume that
the reliability $p$ and the trust probabilities $q\left(  j\right)  $ for all
branch nodes $j$ other than $X$ are fixed arbitrarily. Label the arcs out of
$X$ other than $XH$ as $i=1,2,\dots,n-1.$ Let $\beta_{i}$ denote the expected
time to return to $X$ when leaving $X$ via arc $i.$ Then the expected time
$T\left(  X,H\right)  $ to reach $H$ from $X$ is the same as for the star at
$X$ with rays of length $\beta_{i}/2.$ Regardless of the values of $q\left(
j\right)  $ and the arc lengths $\beta_{i},$ the value of $q=q\left(
X\right)  $ which minimizes $T\left(  X,H\right)  $ is given by $q=$ $\bar
{q}_{n}\left(  p\right)  $ as in (\ref{qn bar star}) and the least expected
time to reach $H$ is given by
\begin{align}
\hat{T}_{H}  & =\hat{T}\left(  X,H\right)  =c+M\beta,\text{ where }\beta
=\sum\limits_{i=1}^{n-1}\beta_{i}\text{ and}\label{T=c+MA}\\
M  & =M\left(  n,p,q\right)  =\frac{\allowbreak\left(  p-2q+q^{2}%
+nq-npq\right)  }{q\left(  1-q\right)  \left(  n-1\right)  }.\label{M(n,p,q)}%
\end{align}

\end{theorem}

\begin{proof}
The same derivation used for Theorem \ref{star} holds in this situation, with
$\alpha_{i}$ replaced by $\beta_{i}/2.$ If $d$ points towards $H,$ we have%
\[
T=qc+\frac{\left(  1-q\right)  }{n}\sum\limits_{i=1}^{n-1}\left(  T+\beta
_{i}\right)  ,
\]
which is the same as (\ref{TH}) with $\beta_{i}$ replacing $2\alpha_{i}.$ The
same replacement holds for the equation (\ref{T_i}) giving $T^{i},$ so the
rest of the analysis follows in an identical fashion. Thus $M$ is half the
constant given in (\ref{Tstar}).
\end{proof}

It is worth noting that the Star Theorem is a special case of the Bridge
Theorem with returns times twice the lengths of the leaf arcs. We can use the
Bridge Theorem to give a simpler solution of the circle-with-spike network of
Figure 5. Optimal trust at $I=X$ (namely $\bar{q}_{3}\simeq0.55051$) now
follows from the Bridge Theorem as $\bar{q}_{3}.$ Since there is a unique node
sequence $AXH$ to the home node, we have $T\left(  A,H\right)  =T\left(
A,X\right)  +T\left(  X,H\right)  $ (see Lemma 4 below). Suppose $p>1/2$ for
simplicity. Then a simple argument given earlier shows that the optimal trust
at node $A$ is $1,$ which implies that the expected return time to $X$ when
leaving via the arc of length $\lambda=1,2$ is given by $\beta_{\lambda
}=\lambda+p\left(  1\right)  +\left(  1-p\right)  (2)=\allowbreak\lambda-p+2$
with sum $\beta=7-2p.$ We can also obtain the time $\hat{T}\left(  X,H\right)
$ evaluated as $5.056$ in Section 3.2 by the slow method using ( \ref{T=c+MA})
as
\[
\hat{T}\left(  X,H\right)  =1+\frac{\allowbreak\left(  p+\bar{q}_{3}+\bar
{q}_{3}^{2}-3p\bar{q}_{3}\right)  }{2~\bar{q}_{3}\left(  1-\bar{q}_{3}\right)
}\left(  7-2p\right)  \simeq5.056,\text{ for }p=3/4.
\]
The nice thing about the Bridge Theorem is that the optimal trust probability
at $X$ only depends on $p$ and $n$. Therefore, if there are multiple nodes
connected to a bridge on different parts of the network of the same degree,
they have the same optimal trust probability.

We conclude this section with a formalization of the claim about $T\left(
A,H\right)  =T\left(  A,X\right)  +T\left(  X,H\right)  $ mentioned in the
previous paragraph.

\begin{lemma}
Let $A,B,C$ be nodes of $Q$ such that every path from $A$ to $C$ passes
through $B$. Suppose reliability $p$ is fixed as well as the trust
probabilities in $\left(  0,1\right)  $ at every node of $Q.$ Then%
\begin{equation}
T\left(  A,C\right)  =T\left(  A,B\right)  +T\left(  B,C\right)
.\label{T(A,C)sum}%
\end{equation}

\end{lemma}

\begin{proof}
From the assumptions, there is a finite state Markov chain on the nodes of
$Q,$ with $C$ as an absorbing state. Almost every sample path starting at $A$
reaches $C$ and the expected hitting time is finite. For every sample path
there are times $t_{1}^{d}$ from $A$ to $B$ and $t_{2}^{d}$ from first arrival
at $B$ to $C,$ with total time from $A$ to $C$ given by%
\[
t^{d}=t_{1}^{d}+t_{1}^{d}.
\]
Since expectation respects summation we have
\[
T^{d}\left(  A,C\right)  =T^{d}\left(  A,B\right)  +T^{d}\left(  B,C\right)  .
\]
Taking expectations with respect to the finite space of direction vectors $d$
and the measure $\mu,$ we similarly have (\ref{T(A,C)sum}).
\end{proof}

Note that the analog of (\ref{T(A,C)sum}) for \textit{optimal} times $\hat{T}
$ (where $q$ might not be the same in the different terms) may be false; we
might have that $\hat{T}\left(  A,C\right)  >\hat{T}\left(  A,B\right)
+\hat{T}\left(  B,C\right)  $ in the event that the last two times are
minimized for different values of $q.$ As an example, consider the tree of
Figure 6, with $p=3/4.$ We showed in our earlier analysis that $\hat{T}\left(
A,H\right)  \simeq8.05$ and $\hat{T}\left(  B,H\right)  \simeq5.23.$ We can
now use the Bridge Theorem to determine $\hat{T}\left(  A,B\right)  $. We take
$X=A$ and $H=B$, $n=2,$ $c=1$ and $\beta=2.$ For $p=3/4,$ this gives $M=.866$
by (\ref{M(n,p,q)}). By (\ref{T=c+MA}) we have $\hat{T}\left(  A,B\right)
=c+M\beta=1+2\left(  .866\right)  =\allowbreak2.\,\allowbreak732.$ So $\hat
{T}\left(  A,B\right)  +\hat{T}\left(  B,H\right)  =2.\,\allowbreak
732+5.23=\allowbreak7.\,\allowbreak962<\hat{T}\left(  A,H\right)  \simeq8.05.$
Of course if in the larger time $T\left(  A,C\right)  $ we are allowed to
choose the optimal trust at every nod, this cannot occur.

\section{Trees}

In the previous section, when applying results on the star to the
circle-with-spike network of Figure 5, we used the fact that certain arcs were
bridges and certain nodes were cuts. These ideas work very well on trees,
where all branch nodes are cuts and all arcs are bridges. For a tree $Q,$ we
choose to view the home (destination) node $H$ as the root. The following
definitions apply to trees. For each node $i\neq H$ there is a unique adjacent
node $s\left(  i\right)  $ which leads to $H,$ called its \textit{successor}.
Similarly there is a set $a\left(  i\right)  =\left\{  j:s\left(  j\right)
=i\right\}  $ of nodes, empty for leaf nodes, which we call the
\textit{antecedents} of node $i.$ Finally, we define the \textit{depth}
$\delta$ of a node $i$ recursively: Leaf nodes nodes have depth $0;$ for other
nodes $i,$ $\delta\left(  i\right)  =1+\max\left\{  \delta\left(  j\right)
,~j\in a\left(  s\left(  i\right)  \right)  \right\}  .$ For every node $A$
there is a unique shortest path to $H.$ By relabeling the nodes, we can write
this path as $j=0,\dots,m,$ with $A$ labeled $0$ and $H$ labeled $m,$ with
$j+1=s\left(  j\right)  .$ Since all nodes are cuts, we can write by repeated
application of (\ref{T(A,C)sum}), the expected time from $A$ to $H$, with the
notation $S\left(  j\right)  =T\left(  j,j+1\right)  ,$ as
\begin{equation}
T\left(  A,H\right)  =T\left(  0,m\right)  =\sum\limits_{j=0}^{m-1}T\left(
j,j+1\right)  =\sum\limits_{j=0}^{m-1}S\left(  j\right)  .\label{t=sumSj}%
\end{equation}
For any node $i,$ since the arc $i,s\left(  i\right)  $ is a bridge, we use
Theorem \ref{localstar} with $H=s\left(  i\right)  $ to write that
\begin{align}
S\left(  i\right)   & =T\left(  i,s\left(  i\right)  \right)  =\lambda\left(
i,s\left(  i\right)  \right)  +M\beta\nonumber\\
& =\lambda\left(  i,s\left(  i\right)  \right)  +2M\sum\limits_{j\in a\left(
i\right)  }\left(  \lambda\left(  j,s\left(  i\right)  \right)  +S\left(
j\right)  \right)  ,\label{Sj recursion general trees}%
\end{align}
taking $n=k\left(  i\right)  ,$ the degree of $i,$ in the definition
(\ref{M(n,p,q)}) of $M=M\left(  n,p,q\right)  $ and recalling that
$\lambda\left(  i,j\right)  $ is the length of that arc. If we are considering
the counting searcher problem we take trust $\hat{q}_{k\left(  i\right)  }$
for each node $i,$ and the recursion (\ref{Sj recursion general trees}) solves
the problem, starting with leaf nodes and increasing the depth. Thus we have
shown the following.

\begin{theorem}
Let $Q$ be a tree network. The counting searcher problem is solved by taking
$\hat{q}_{k}$ equal to $\bar{q}_{k}$ as defined for the star in
(\ref{qn bar star}). For an $n-$ary tree, where all branch nodes have $n$
antecedents (and degree $n+1)$ the general solution is $\hat{q}=\bar{q}_{n+1}$
(at all nodes). So for a binary tree the solution to the Satnav problem is
$\hat{q}=\bar{q}_{3}$ and for the line graph the solution is similarly
$\hat{q}=\bar{q}_{2}.$
\end{theorem}

\section{Time to Cross a Line}

We now consider the case on a line graph with nodes $0,1,\dots,n$ (or even on
a one sided infinite line graph). We calculate the optimal time taken from a
node to a larger node. We do this first with variable length arcs and then
specialize to unit length arcs (graphs). We show that for $p\neq1/2$ the time
to reach $n$ is linear in $n,$ but if $p=1/2$ it is quadratic.

\begin{theorem}
Let $Q$ be a line graph on nodes $0,1,\dots$ where the length $\lambda\left(
i,i+1\right)  $ of the arc between $i$ and $i+1$ is denoted $a_{i}$ and
$b_{j}\equiv\sum_{i=0}^{j-1}a_{i}.$ Let $p\neq1/2.$ Then the least expected
time $S\left(  j\right)  =T\left(  j,j+1\right)  $ from node $j$ to node $j+1$
is given by
\begin{align}
S\left(  j\right)   & =a_{j}+2\sum_{i=1}^{j}a_{j-i}~z^{i},\text{ }%
j\geq0~\text{and }S\left(  0\right)  =a_{0},\text{where}\label{S(j)}\\
z  & =\frac{q^{2}-2pq+p}{q\left(  1-q\right)  }.\label{z}%
\end{align}
The least expected time $T\left(  0,j\right)  $ to reach node $j$ from the
leaf node $0$ is given by%
\begin{equation}
T\left(  0,j\right)  =b_{j}+2\sum_{i=1}^{j-1}b_{j-i}~z^{i}.\label{T(0,j)}%
\end{equation}
All of these times are minimized by taking the trust probability $q$ to be
$\bar{q}_{2}\left(  p\right)  =\hat{q}=\frac{p-\sqrt{p\left(  1-p\right)  }%
}{2p-1},$ which makes $z=2\sqrt{p\left(  1-p\right)  }\rightarrow0$ as
$p\rightarrow1$ (or to $0$). Hence as the reliability goes to $1,$ the time to
cross a line converges to its length $b_{j}$. We have%
\[
T\left(  0,j\right)  <b_{j}\left(  1+2\sum_{i=1}^{j}~z^{i}\right)
=b_{j}\left(  1+\frac{2\left(  z-z^{j+1}\right)  }{1-z}\right)  ,
\]
so the crossing time is linear in the length of the line.

If $p=1/2$ the optimal $q=\hat{q}=1/2,$ giving a random walk on the line,
where the crossing time is quadratic.
\end{theorem}

\begin{proof}
To obtain the formula (\ref{S(j)}) for the incremental times $S\left(
j\right)  $ we note that $S\left(  0\right)  =a_{0}$ because $0$ is a leaf
node. To obtain a formula for $S\left(  j\right)  $ in terms of $S\left(
j-1\right)  ,$ we apply the Bridge Theorem, Theorem 3, with $H=j+1$ and $X=j,$
$c=a_{j}.$ This gives $q=\bar{q}_{2}$ and gives $S\left(  j\right)  =T\left(
X,H\right)  $ as
\[
T\left(  X,H\right)  =S\left(  j\right)  =a_{j}+\beta M.
\]
The expected return time $\beta$ when leaving node $j$ by the arc $\left(
j-1,j\right)  $ is given by $a_{j-1}$ plus the expected time to return to $j$
from $j-1,$ which by definition $S\left(  j-1\right)  .$ Finally, the general
formula for $M$ in (\ref{M(n,p,q)}) simplifies to the number $z$ in (\ref{z})
when $n=2.$ Thus we have the recursion
\begin{equation}
S\left(  j\right)  =a_{j}+\left(  a_{j-1}+S\left(  j-1\right)  \right)
z,\text{ with }S\left(  0\right)  =a_{0}.\label{recursiona}%
\end{equation}
To check the formula inductively, we write%
\begin{align*}
\left(  a_{j-1}+S\left(  j-1\right)  \right)  ~z  & =a_{j-1}z+\left(
a_{j-1}+2\sum_{i=0}^{j-1}a_{j-i}~z^{i}\right)  z\\
& =a_{j-1}z+\left(  a_{j-1}z+2\sum_{i=0}^{j}a_{j-i}~z^{i}\right) \\
& =S\left(  j\right)  -a_{j}.
\end{align*}
The formula for $T\left(  0,j\right)  $ given in (\ref{T(0,j)}) now follows
from the Cut Lemma (Lemma 4) because every node $i>0$ is a cut node.
\end{proof}

If all the arcs have unit length $a_{i}=1$ then we have $b_{j}=j$ and then the
formula (\ref{T(0,j)}) can be simplified.

\begin{corollary}
If $Q$ is a line network with unit length arcs then, for $p\neq1/2,$ the
optimal time to cross it is given by
\begin{equation}
T\left(  0,j\right)  =\frac{j-jz^{2}+2z\left(  z^{j}-1\right)  }{\left(
1-z\right)  ^{2}}\label{Tunit}%
\end{equation}
If $p=1/2,$ the optimal time is the expected time for a random walk to reach
node $j$ from node $0,$ that is, $j^{2}.$ Obviously this is quadratic rather
than linear in the length $j$ of the line. (For $p\neq1/2$ we showed this time
is linear in $j,$ in a more general context.)
\end{corollary}

We plot below in Figure 9 the optimal expected travel times from the left leaf
node $0$ to node $j.$ For $p=0$ or $1,$ $T\left(  0,j\right)  =j,$ a direct
path can always be taken. Note that the expected travel times $\bar{S}\left(
j\right)  =T\left(  j-1,j\right)  $ between consecutive nodes is increasing.%
    \begin{figure}[H]
     \centering
     \includegraphics[width=0.95\textwidth]{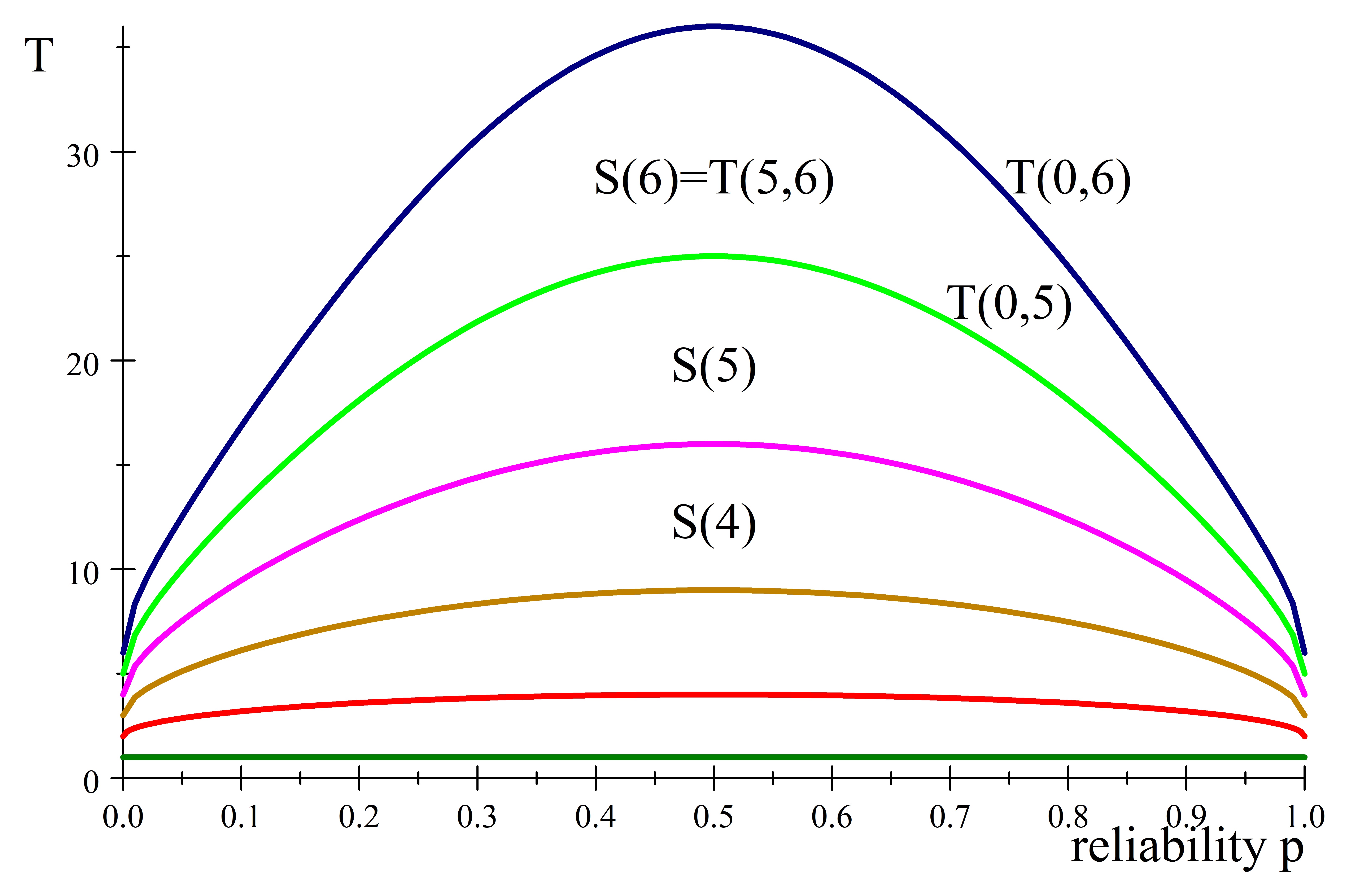}
     \caption{Plots of $T\left(  0,j\right)  ,$ $j=1$ to $6,$ gaps $S\left(
j\right)  .$}
    \end{figure}
Note that it takes longer to traverse consecutive nodes $j-1$ and $j$ as $j$
increases. It is interesting to notice the asymmetry of travel times, with is
not present in traditional shortest path problems. Observe that for $i<j$ we
have
\begin{equation}
T\left(  0,i\right)  +T\left(  i,j\right)  =T\left(  0,j\right)  ,\text{ or
}T\left(  i,j\right)  =T\left(  0,j\right)  -T\left(  0,i\right)
.\label{decomp}%
\end{equation}
For example when $p=3/4$ we have $T\left(  3,5\right)  \simeq12.187,$ using
(\ref{decomp}) and (\ref{Tunit}). Note that this doesn't depend on the total
length of the line graph, as 5 becomes an absorbing state. However if $i>j$
then the time for $T\left(  i,j\right)  $ (going left) does depend on the
length $0,1,\dots,n$ of the line graph. Looking at it so that node $n$ is on
the left, our earlier analysis shows that for $n=7$ we have
\begin{equation}
T\left(  5,3\right)  =T\left(  n-5,n-3\right)  =T\left(  2,4\right)
\simeq9.763\label{asymm}%
\end{equation}
This is clearer if we take an extreme situation with $Q$ having nodes $0$ to
$100.$ If we want to go from $1$ to $2,$ at most we can backtrack to $0.$ If
we want to go from $2$ to $1,$ if we are unlucky we may travel very far to the
right before reaching $1.$Note: Travel times from the leaf node $0$ are
greater than those from node $1$ by one. It is a matter of taste whether to
give a formula for $T\left(  1,j\right)  $ or for $T\left(  0,j\right)  .$ We
plot below the optimal expected travel times from the left leaf node $0$ to
node $j.$ For $p=0$ or $1,$ $T\left(  0,j\right)  =j,$ a direct path can
always be taken. Note that the expected travel times $\bar{S}\left(  j\right)
=T\left(  j-1,1\right)  $ between consecutive nodes is increasing.

\section{Cycle Graphs}

In this section we analyze the Satnav Problem on the cycle graphs $C_{3}$ and
$C_{4}$ of Figure 10. We believe these represent the cases where there are an
odd or even number of nodes. In the latter case there is an antipodal node
(called $C)$ to the home node $H.$ Note that $C_{4}$ is an example with
non-unique shortest paths to $H.$ So the direction at $C$ is equiprobable. The
quick general methods used on line graphs do not appear to help the analysis
for cycles, so this section is really just an introduction to the general
problem. The Satnav problem on $C_{n}$ is identical to the destination
\textit{set} problem on $L_{n+1},$ the graph 0-1-2-...-n where the problem is
to reach the set $\left\{  0,n\right\}  $ from a given node $i,$ $0<i<n$.
However we find significant qualitative differences in the solution, for
example on $C_{4}$ there is no uniformly optimal trust probability, this
depends on the starting node. This is in contrast with the Satnav Problem on
$L_{n},$ where we found a uniformly optimal trust probability $\hat{q}.$ On
the other hand, optimal travel times on a cycle are clearly symmetric, unlike
the situation found for the line at the end of Section 7.\\
    \begin{figure}[H]
     \centering
     \includegraphics[width=0.85\textwidth]{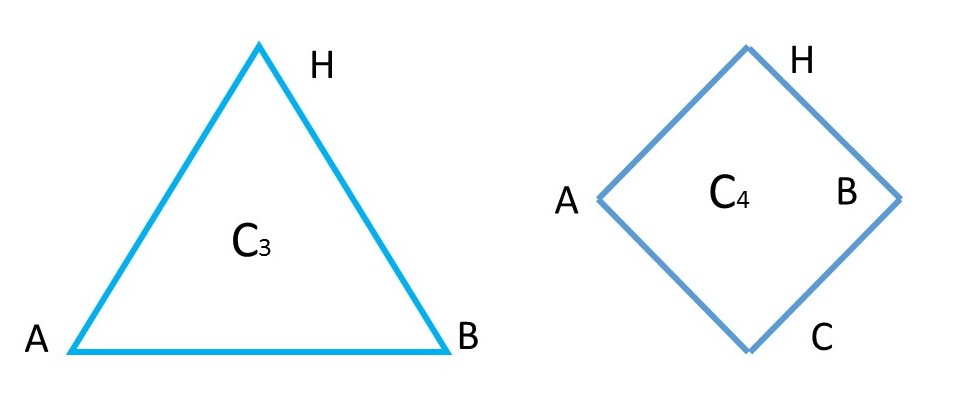}
     \caption{The cycle networks $C_3$ and $C_4.$}
    \end{figure}

\subsection{The cycle $C_{3}$}

We adopt the `slow method' used in Section 3 for the case $x=1$. This
involves, for each of the four direction vectors $d^{i}$ on the two branch
nodes, construction the two simultaneous equations for the expected time $a$
and $b$ to reach $H$ starting from $A$ and $B,$ respectively. Compared with
the solution for the triangle with $x>2$ given in Section 3.1, the doubly
correct vector is now $d^{4}=\left(  -,+\right)  $ which occurs with
probability $p^{2}.$ Recalculating the time $a_{i}=T_{A}^{d_{i}}=T^{d_{i}%
}\left(  A,H\right)  $ and averaging over the probabilities $\mu\left(
d_{i}\right)  $, we get%

\begin{align}
T_{A}  & =p^{2}~a_{4}+p\left(  1-p\right)  ~a_{2}+\left(  1-p\right)
p~a_{1}+\left(  1-p\right)  ^{2}~a_{3},\text{ giving }\\
T_{A}  & =\frac{p^{2}}{q}-\frac{3p\left(  p-1\right)  }{q^{2}-q+1}%
+\frac{\left(  1-p\right)  ^{2}}{1-q}\text{ for }x=1.\text{ }%
\end{align}
We obtain an implicit function of $\hat{q}$ as a function of $p$ by simply
setting the partial derivative of $T_{A}$ with respect to $q$ equal to zero.
We plot this implicitly in Figure 11. The symmetry of $A$ and $B$ means that
this is also the optimal trust when starting at $B.$ So there is a uniformly
optimal trust function In particular for our standard reliability $p=3/4,$ the
optimal trust is approximately $0.786\,76$ as seen in Figure 11. Note the
difference from the case $x=3$ of Section 3.1.wq23\\ \\ \\
\begin{minipage}[c]{0.5\linewidth}
    \begin{figure}[H]
     \centering
     \includegraphics[width=0.95\textwidth]{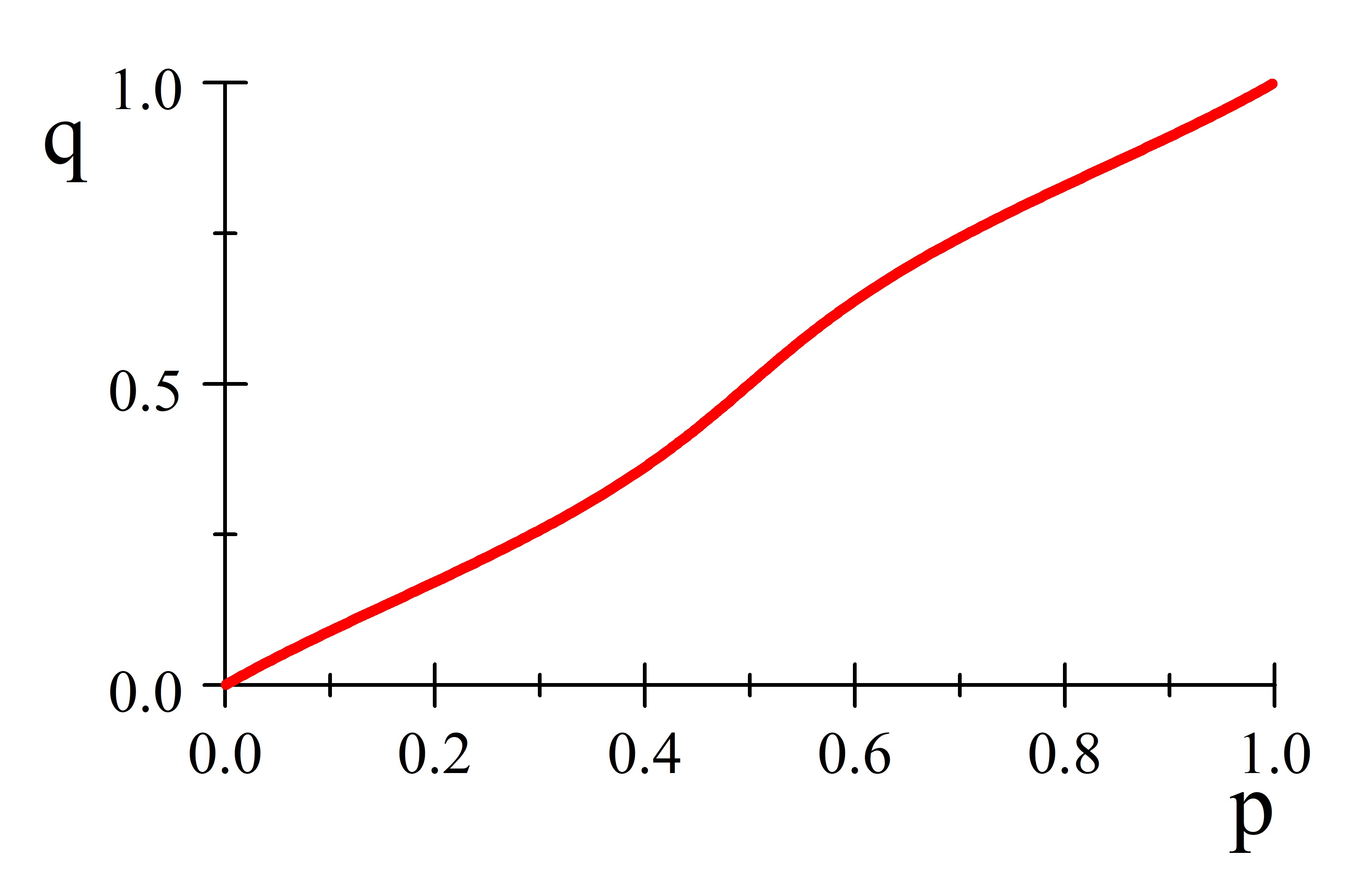}
     \caption{Optimal trust $\hat{q}.$}
    \end{figure}
\end{minipage}
\begin{minipage}[c]{0.5\linewidth}
    \begin{figure}[H]
     \centering
     \includegraphics[width=0.95\textwidth]{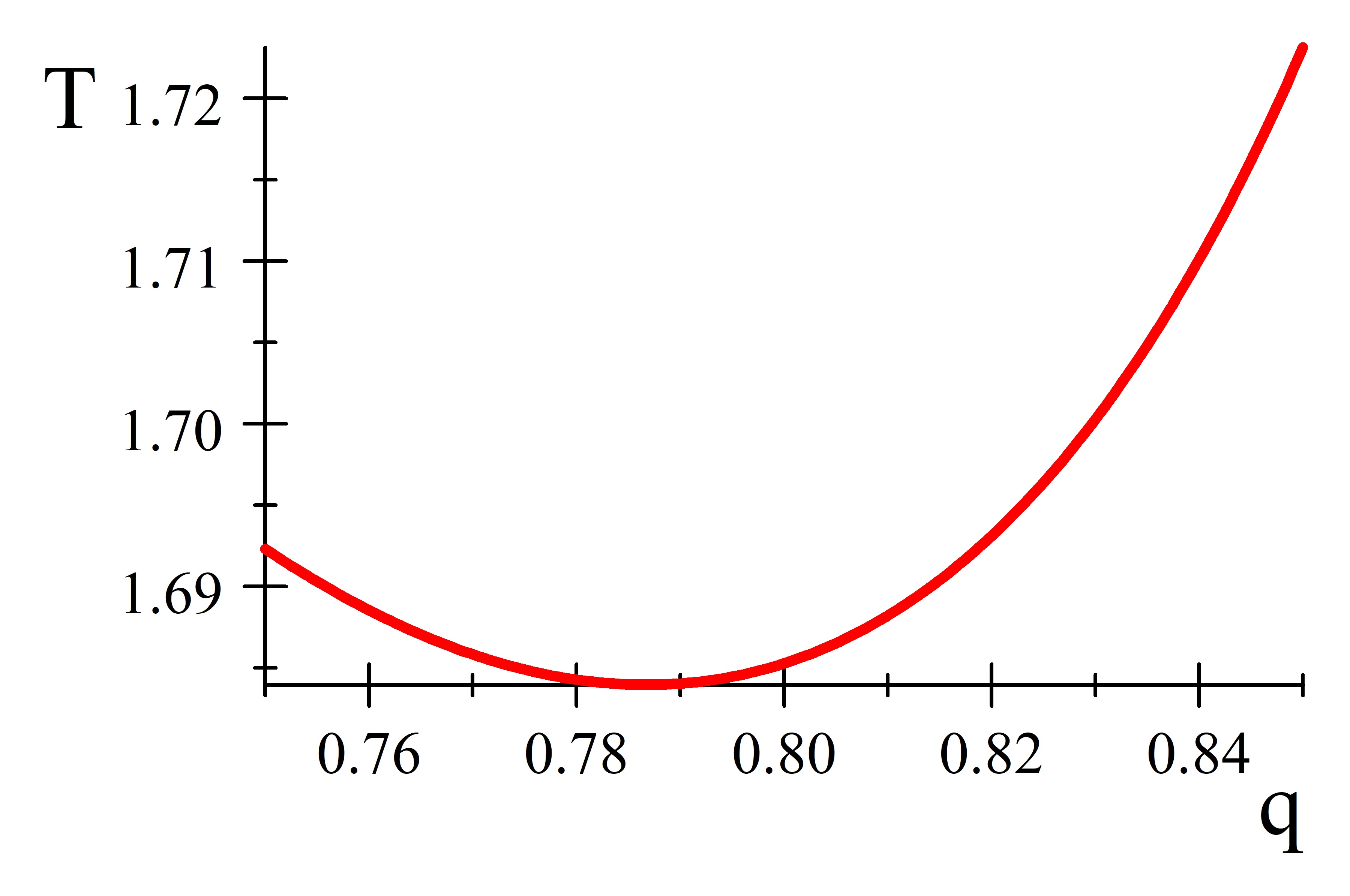}
     \caption{Plot of $T_A\left(  3/4,q\right)  .$}
    \end{figure}
\end{minipage}\newpage
\subsection{The cycle $C_{4}$}

The slow method for solving the Satnav Problem on $C_{4}$ has equation systems
for each of the eight direction vectors on the branch nodes $A,C,B.$ Note that
regardless of the reliability $p,$ the direction at $C$ is equally likely
towards $A$ or $B.$ For example, when all pointers are in the clockwise
direction, $d=\left(  +,+,+\right)  ,$ we have the system (where $a$ is the
expected time from $A$ to $H,$ same for $b$ and $c)$%

\begin{align*}
a  & =\left(  1-q\right)  \left(  1\right)  +q\left(  1+b\right)  \\
c  & =\left(  1-q\right)  \left(  a+1\right)  +q\left(  b+1\right)  \\
b  & =\left(  1-q\right)  \left(  1+c\right)  +q\left(  1\right)
\end{align*}
Using the same methods as for $C_{3},$ we can implicitly plot (see Figure 13)
the optimal trust $\hat{q}$ at all nodes when starting at $A$ (or $B),$ the
lower red curve, and when starting at $C$ (the higher green curve). The
important observation is that the cycle $C_{4},$ unlike the line graphs or the
odd cycle $C_{3},$ does \textit{not} have a uniform trust solution, the
optimal trust depends on the starting node.
    \begin{figure}[H]
     \centering
     \includegraphics[width=0.95\textwidth]{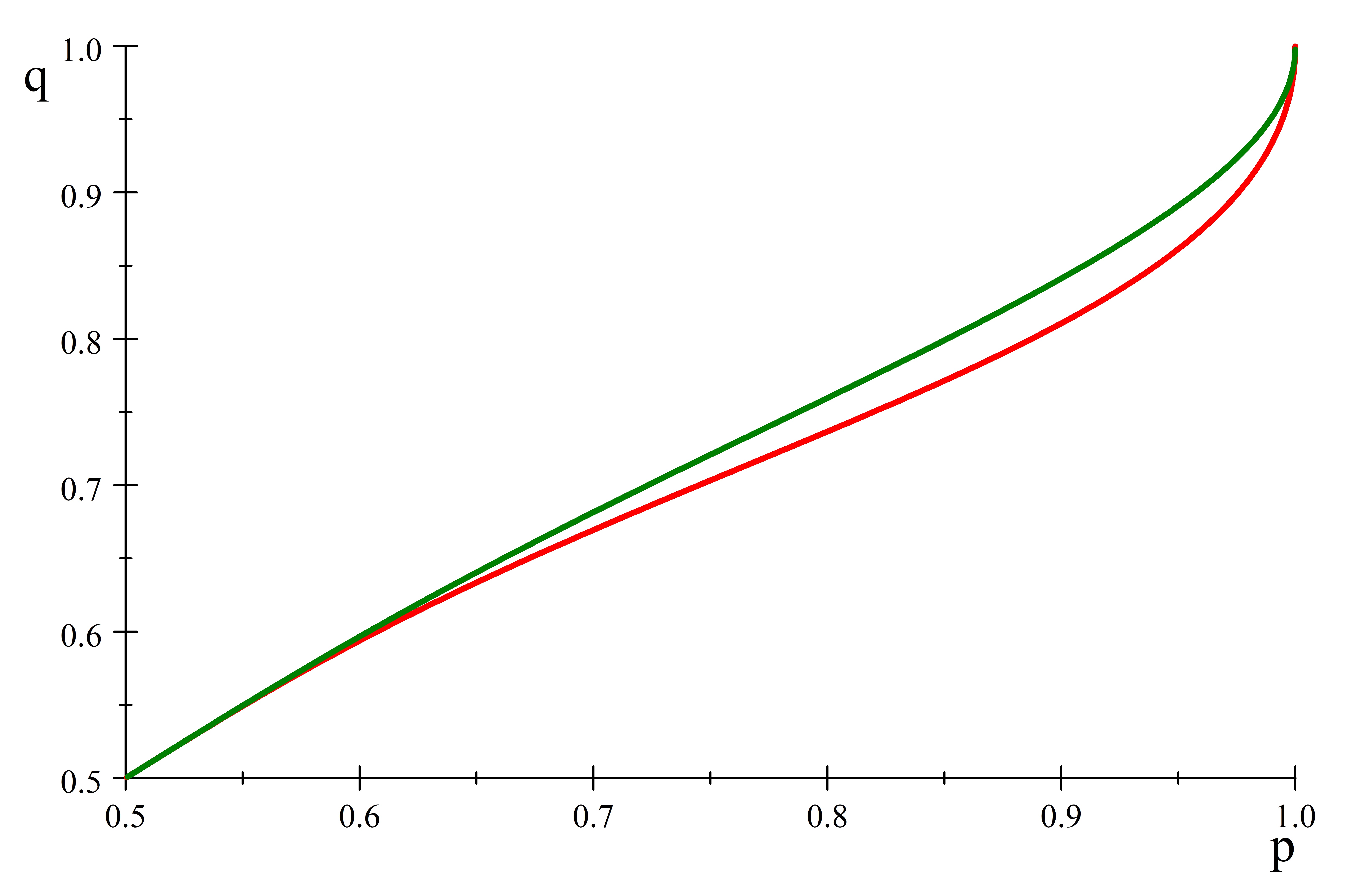}
     \caption{Circle $C_4:$ $\bar{q}\left(  p\right)  $ from $A$ (red, lower),
$C$ (green).}
    \end{figure}
\section{First To Nest Wins (Treasure Hunting)}

We consider a two-person constant sum game where the first player to reach the
Home node $H$ wins, and if they reach at the same time the winner is
determined by a fair coin toss. The payoff is the probability that Player I
wins. Since game problems are much harder than individual optimization, with
take the simplest nontrivial network, the line with three nodes, $0,1,2, $
with $H=2.$ We consider both the symmetric game where both players start at
node $1$ and the asymmetric game where they start at $0$ and $1.$ Note that
this network is also the star with three nodes. So the individual time
minimization problem has been solved earlier in two ways (star and line). We
note that this is a winner-take-all game in that each player gets a score (the
hitting time to $H$) and the lowest score wins. Such games have been analyzed
in Alpern and Howard (2018), but this version is not covered by any theory in
that paper. Both players have the same satnav (the same pointer at node $1$),
which is correct with probability $p$. Player I trusts with prob $q$, II with
prob $r$. An alternative model, not analyzed here, is for the two players to
have different Satnavs, with independent errors. In this case the game fits
exactly into the Alpern-Howard scenario.

\subsection{Symmetric Start}

Here we assume that both players start at node $1,$ so we know the value
(assuming it exists - it does) must be $1/2.$ If $d=+,$ pointer correctly
points to $2,$ the payoff $v^{+}$ satisfies the following, recalling that a
tie in reaching node $2$ has payoff $1/2.$%
\begin{align*}
v^{+}  & =qr\left(  1/2\right)  +q\left(  1-r\right)  \left(  1\right)
+\left(  1-q\right)  r\left(  0\right)  +\left(  1-r\right)  \left(
1-q\right)  \left(  v^{+}\right)  ,\text{ so}\\
v^{+}  & =\frac{2q-qr}{2q+2r-2qr}\text{ and similarly }v^{-}=\frac
{1-q+r-qr}{2\left(  1-qr\right)  }.
\end{align*}
It follows that%
\[
v\left(  p,q,r\right)  =p\frac{2q-qr}{2q+2r-2qr}+\left(  1-p\right)
\frac{1-q+r-qr}{2\left(  1-qr\right)  }.
\]
Solving the equation$\frac{\partial v\left(  p,q,r\right)  }{\partial r}=0$ to
obtain $\hat{r}\left(  p,q\right)  $ and solving $q=$ $\hat{r}\left(
p,q\right)  $ gives
\begin{equation}
\hat{q}_{sym}\left(  p\right)  =\frac{-1+p+\sqrt{1-3p+3p^{2}})}{2p-1}%
.\label{qsym2}%
\end{equation}
Thus we have shown the following.

\begin{theorem}
The optimal trust in the symmetric game on the line graph $\left\{
0,1,2=H\right\}  $ where both players start at node $1$ is given by $\hat
{q}_{sym}\left(  p\right)  $ as in (\ref{qsym2}).
\end{theorem}

Figure 14 shows the intersection of the optimal response curves when $p=2/3$
at $\hat{q}_{sym}\left(  2/3\right)  =\allowbreak\sqrt{3}-1=\allowbreak
0.732\,05$ for both players.%
    \begin{figure}[H]
     \centering
     \includegraphics[width=0.95\textwidth]{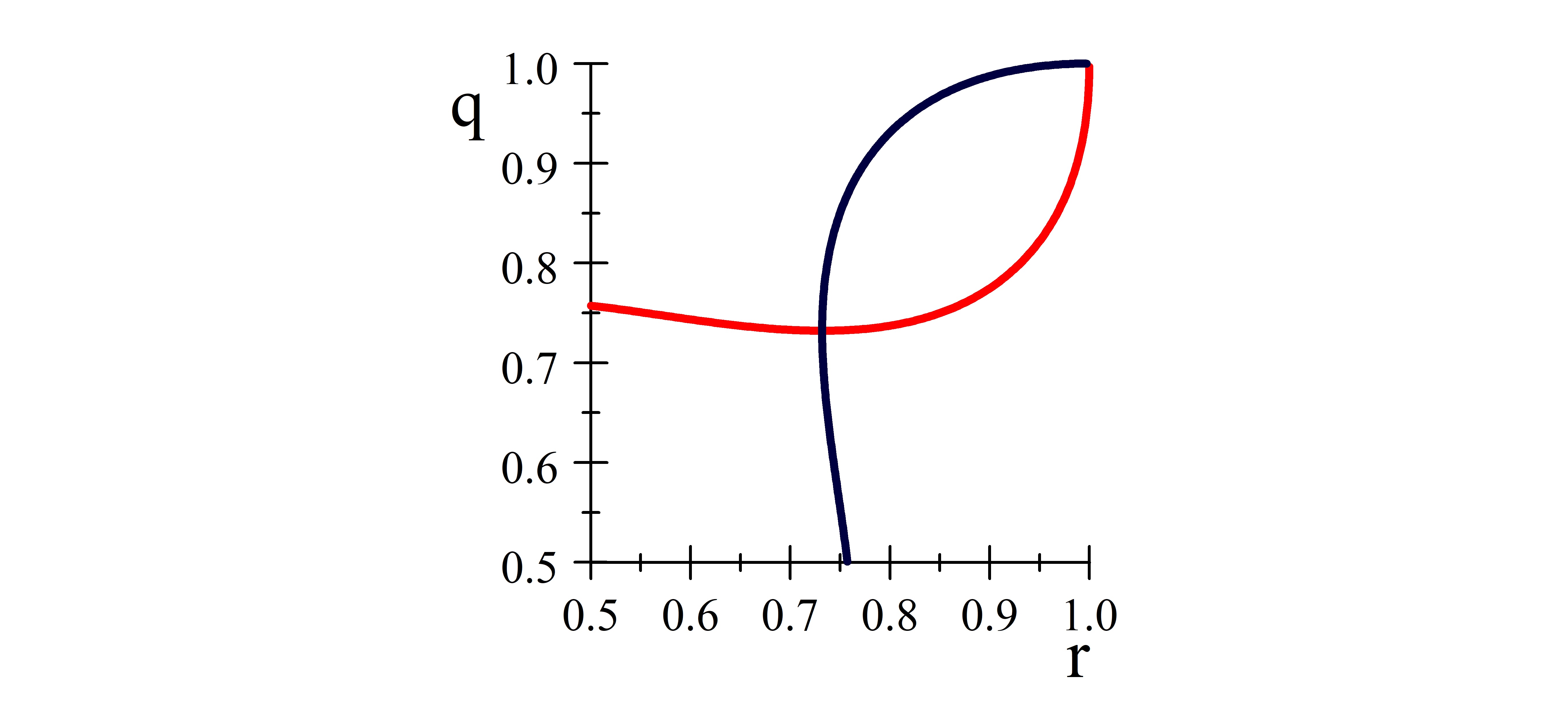}
     \caption{Plots of $q=\hat{r}\left(  r\right)  $ (red), $r=\hat{r}\left(
q\right)  $ (blue).}
    \end{figure}
\subsection{Asymmetric Start}

We now consider the scenario where Player I starts at node $1$ (at time $1)$
and Player II starts at node $0.$ Note that as long as the game is being
played, this will be the position at all odd times, and at all even times
Player II will be at node 1 and Player I will be at node $0.$ There cannot be
a tie. We take $q$ as I's trust and $r$ as II's trust.

\begin{theorem}
Consider the game on the line with node set $\left\{  0,1,2=H\right\}  $ where
first to $H$ wins. Suppose player I starts at node 1 and player II starts at
node 0. It is optimal for player II to follow a random walk, that is, trust
$r=1/2.$ For player I there are three cases.

\begin{enumerate}
\item If the reliability satisfies $p\geq4/5,$ I's optimal trust is $\hat
{q}=1,$ so the value is $v=p.$ (Either I goes immediately to $H$ and wins or
he oscillates between $1$ and $0$ and loses.)

\item If $1/2<p\leq4/5,$ then $\hat{q}=Q\left(  p\right)  =\left(
1+p-3\sqrt{p\left(  1-p\right)  }\right)  /\left(  2p-1\right)  .$ Player I
wins with probability (value)
\[
v=\left(  4/3\right)  \left(  1-\sqrt{p\left(  1-p\right)  }\right)  .
\]

\item If $p=1/2$ then both player I also optimally follows a random walk and
wins with probability
\[
v=\frac{1}{2}+\frac{1}{8}+\frac{1}{32}+\dots+~=\frac{1/2}{1-1/4}=\frac{2}{3}%
\]

\end{enumerate}
\end{theorem}

\begin{proof}
Let $v=v\left(  p,q,r\right)  $ denote the payoff (probability I wins)when
player I is at 1 and $w$ denote the payoff (probability that I wins) when
player II is at 1. As above, I trusts with probability $q,$ II with
probability $r.$ If $d=+$ (pointer at node 1 is correct, to right) then we
have
\begin{align*}
v^{+}  & =q\left(  1\right)  +\left(  1-q\right)  \left(  w^{+}\right)
,~w^{+}=r\left(  0\right)  +\left(  1-r\right)  \left(  v^{+}\right)
,\text{so}\\
v^{+}  & =\frac{q}{q+r-qr}\text{ and }w^{+}=\frac{q-qr}{q+r-qr}.
\end{align*}
Similarly$\allowbreak$ if $d=-$ (points to left, to $0),$ we have%
\begin{align*}
v^{-}  & =q\left(  w^{-}\right)  +\left(  1-q\right)  \left(  1\right)
,~w^{-}=r\left(  v^{-}\right)  +\left(  1-r\right)  \left(  0\right) \\
v^{-}  & =\frac{1-q}{1-qr},w^{-}=\frac{r-qr}{1-qr}%
\end{align*}
This gives the payoff (winning probability) for Player I, when starting at
node $1,$ as%
\begin{equation}
v\left(  p,q,r\right)  =p\frac{q}{q+r-qr}+\left(  1-p\right)  \frac{\left(
1-q\right)  }{1-qr}.\label{vasym}%
\end{equation}
For those preferring a more probabilistic coin tossing derivation of equation
(\ref{vasym}), consider that I and II have coins which come up heads with
respective probabilities $a$ and $b,$ and (starting with I) they alternate
tossing until one of them gets heads and wins. The probability that I wins on
the $2i+1$th toss is $\left(  \left(  1-a\right)  \left(  1-b\right)  \right)
^{i}a.$ So the probability that I wins is given by%
\[
a\sum_{1}^{\infty}\left(  \left(  1-a\right)  \left(  1-b\right)  \right)
^{i}=\frac{a}{1-\left(  1-a\right)  \left(  1-b\right)  }.
\]
If the pointer is correct, +, then the probabilities of going to node $2=H$
and winning for I and II when at node $1$ are given by $a^{+}=q,b^{+}=r$ and
when pointer is incorrect, they are $a^{+}=1-q,b^{+}=1-r.$ So the probability
that I wins is given by%
\[
p~\frac{a^{+}}{1-\left(  1-a^{+}\right)  \left(  1-b^{+}\right)  }+\left(
1-p\right)  \frac{a^{-}}{1-\left(  1-a^{-}\right)  \left(  1-b^{-}\right)  },
\]
which simplifies to (\ref{vasym}). We now prove the three assertions.

\begin{enumerate}
\item Since $q=1$ guarantees player I wins with probability $p,$ it is enough
to show that a random walk $\left(  r=1/2\right)  $ for player II guarantees
that I wins with probability $\leq p.$ We calculate
\begin{align*}
\frac{\partial v\left(  p,q,1/2\right)  }{\partial q}  & =\frac{2~f\left(
p,q\right)  }{\left(  q-2\right)  ^{2}\left(  q+1\right)  ^{2}},\text{
where}\\
f\left(  p,q\right)   & =-1+5p-2q-2pq-q^{2}+2pq^{2}.
\end{align*}
Since $f\left(  p,q\right)  $ is positive on $4/5<p\leq1,$ $0\leq q\leq1,$ it
follows that $v\left(  p,q,1/2\right)  $ is increasing in $q$ in this range,
so that the best response of player I to $r=1/2$ is $q=1.$ Thus playing
randomly for player II keeps the probability that I wins no more than $p.$

\item In this region of $p,$ the first order equation $f\left(
p,q,1/2\right)  =0$ has the unique probability solution $\hat{q}=Q\left(
p\right)  $ given in the statement. So $Q\left(  p\right)  $ is the optimal
response to $r=1/2.$ The optimal response function for Player II is obtained
by the first order condition%
\[
\frac{\partial v\left(  p,q,r\right)  }{\partial r}=0,\text{ so the optimal
response }\hat{r}=\hat{r}\left(  p,q\right)  \text{ is given by}%
\]%
\[
\frac{2q-2q^{2}+2pq^{2}-\sqrt{\left(  -2q+2q^{2}-2pq^{2}\right)
^{2}-4(p-q^{2}+pq^{2})(-1+p+2q-2pq-q^{2}+2pq^{2})}}{2\left(  -1+p+2q-2pq-q^{2}%
+2pq^{2}\right)  }%
\]
Now fix $p$ and consider Player II's best response to $Q\left(  p\right)  $
for Player I. We find that%
\[
\hat{r}\left(  p,Q\left(  p\right)  \right)  =1/2.
\]
This means that the best response is $r=1/2,$ so $Q\left(  p\right)  $ and 1/2
form an equilibrium.

\item The statement of the Theorem shows an easy way to compute the value of
the game, given that both players adopt a random walk. The optimality of a
trust of 1/2, the random walk, can be obtained by continuity from part 2.
\end{enumerate}
\end{proof}

The alert reader will note that we have avoided the computation of the optimal
response $\hat{q}\left(  r\right)  $ to a Player II strategy of $r.$ In fact
we have derived this response function and we plot the two curves in Figure
15, for $p=2/3,$ with an intersection at $q=Q\left(  2/3\right)
=\allowbreak0.757\,36$ and $r=1/2.$%
    \begin{figure}[H]
     \centering
     \includegraphics[width=0.9\textwidth]{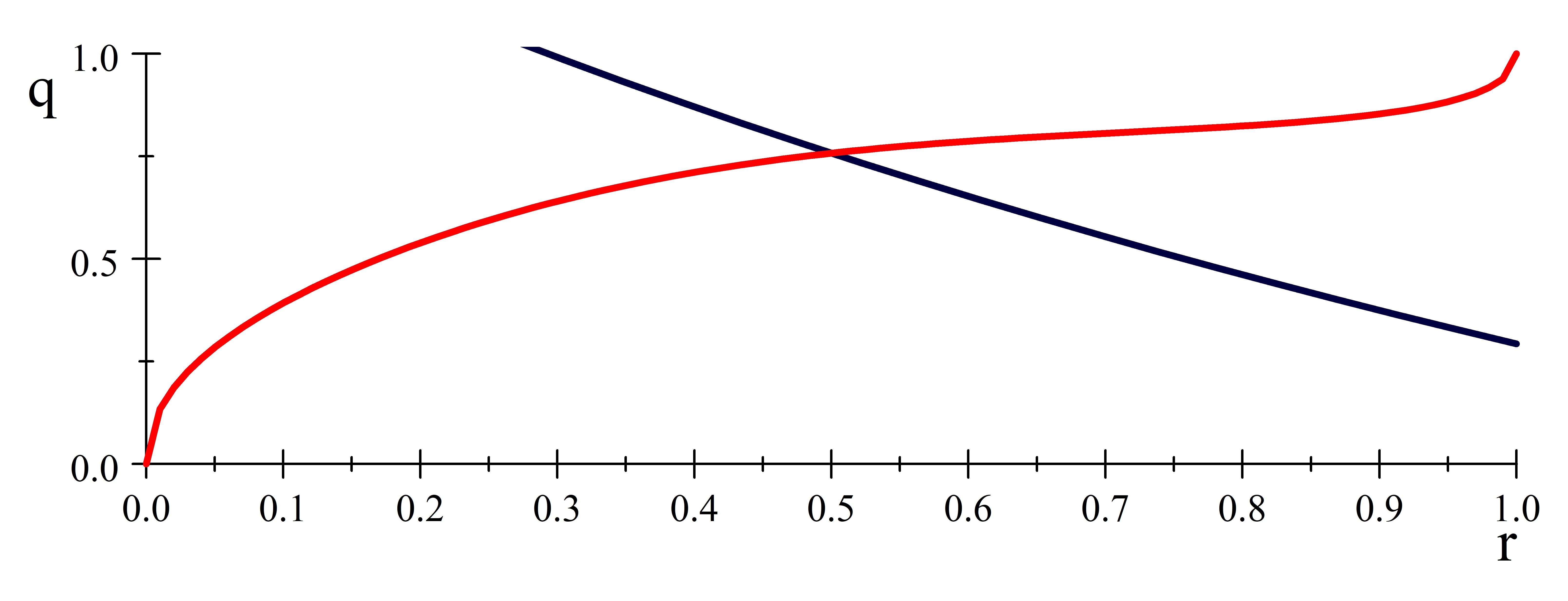}
     \caption{Response curves $\hat{q}\left(  r\right)  $ (red) and $\hat
{r}\left(  q\right)  ,$ $p=2/3.$}
    \end{figure}
    \begin{figure}[H]
     \centering
     \includegraphics[width=0.9\textwidth]{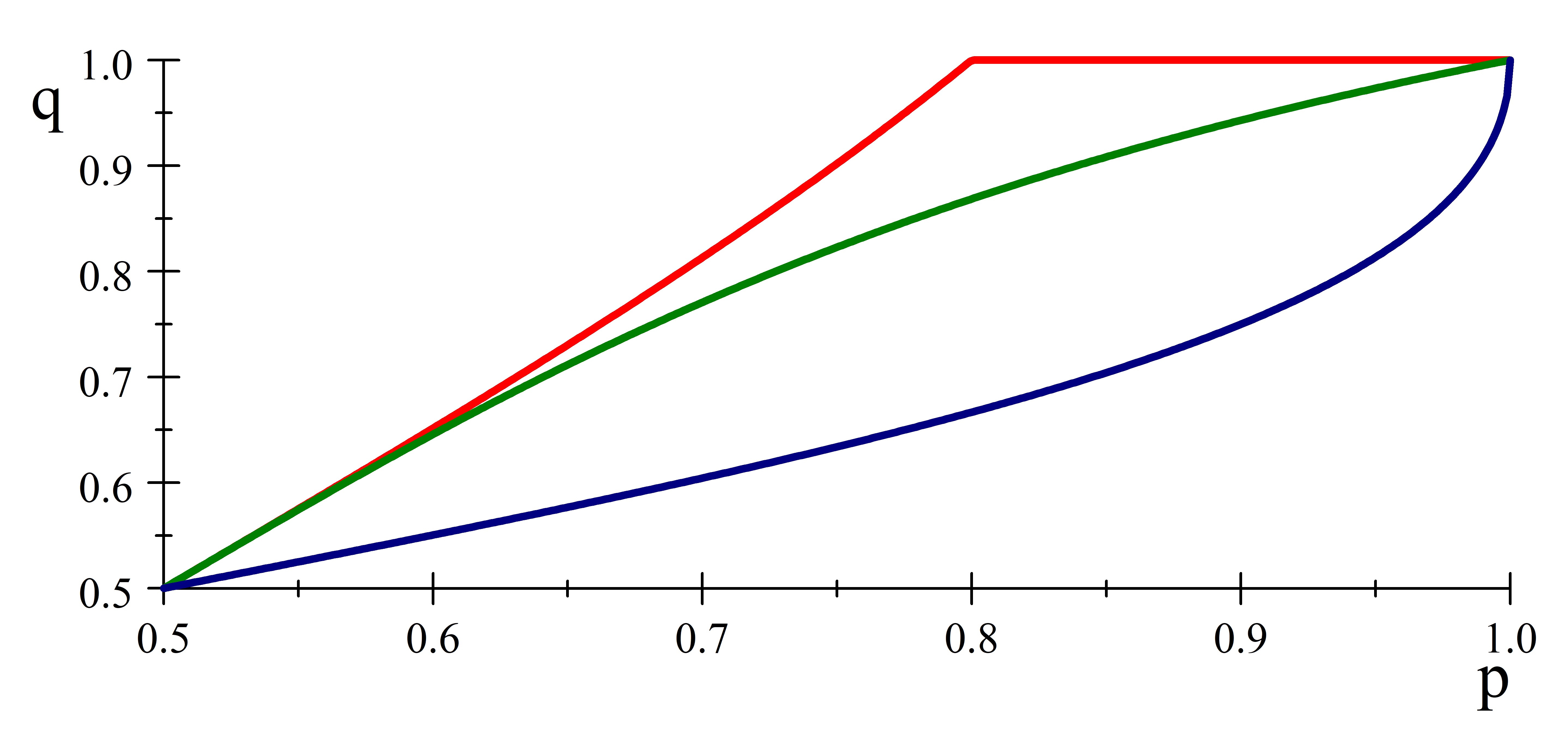}
     \caption{Optimal $q$ for asymmetric game (top), symmetric game, individual
(bottom).}
    \end{figure}

Figure 16 summarizes the optimal trusts for the symmetric and asymmetric
games, compared with an individual who wants to minimize the expected time to
reach node $2$ from node $1.$ It shows the optimal trust for (top) the player
starting at node $1$ in the asymmetric game, (middle) the symmetric game and
(bottom) an individual using $\bar{q}_{2}\left(  p\right)  $ minimizing to
minimize the expected time to get to node $2$ from node 1. The optimal trust
is $1/2$ for Player II in the asymmetric game.

Before leaving the asymmetric game, it is worth giving an intuitive but false
idea for the solution. Note that when Player II is considering his choice of
trust $r,$ he realizes that this value will only be used if and when he gets
to node 1, in which case Player I will be at node $0$ (if the game has not
ended). So in a sense he is in the same position as Player 1 was in at the
start of the game. Consequently, at an equilibrium $r$ should be the same as
$q.$ We have shown this is false, but we leave it up to the reader to find a
flaw in this argument.

\section{Conclusion}

This paper presents a very simple model of finding shortest time paths in
networks with unreliable directional information. We give a simple but slow
method which works on any network and derive some theory which gives quick
solutions for some families of networks. Our model of the search agent is very
simple. He trusts the pointer direction with a chosen probability, possibly
dependent on the degree of the node he is at. More sophisticated agents might
be modeled in the future. For example, it seems reasonable to assume that, in
addition to counting the degree, he can remember which arc he has just arrived
on. Then he can also choose that arc (that is, backtrack) with a different
probability (likely smaller) than the other incident arcs.

We also considered a treasure hunt, where two agents try to be the first to
reach the home node, and to find the treasure. Here, we modeled this problem
in a scenario where both agents (players) have the same pointers, possibly
because they use the same brand of Satnav (GPS). An alternative model which
seems to present interesting facets is to assume they have different brands,
and independent pointers. Additionally, the two brands might have different
reliabilities. Or more generally, the players could have different targets.
This is also a model of what are called `patent races'.

For future work, the model could be modified. For example, instead of a
searcher who seeks a fixed home node, we could have the home node viewed as
another mobile searcher, as in the rendezvous problem of Ozsoyeller et al
(2019). Or the searcher might want to visit a sequence of nodes (rather than
just one) in an effort to patrol the graph against intruders as in Basilico et
al (2017). A similar approach might be taken to deal with other recommendation
systems provided by black box AI processes which are known to be faulty.

\section{Acknowledgement}

The author acknowledges support from the AFIT Graduate School of Engineering
and Management, FA8075-14-D0025.

\end{document}